\documentclass[opre,nonblindrev]{informs3_SSRN}
\RequirePackage{tgtermes}
\RequirePackage{newtxtext}
\RequirePackage{newtxmath}
\RequirePackage{bm}
\RequirePackage{endnotes}

\OneAndAHalfSpacedXI

\usepackage{amsmath}
\usepackage{amssymb}

\usepackage{indentfirst}
\usepackage[hidelinks,hypertexnames=false]{hyperref}
\usepackage{paralist}
\usepackage{geometry}
\usepackage{color}
\usepackage{multimedia}
\usepackage{multicol}
\usepackage{bbm}
\usepackage{graphics}
\usepackage{graphicx}
\usepackage{booktabs}
\usepackage{subcaption}
\usepackage{grffile}
\usepackage{float}
\usepackage{url}

\usepackage{supertabular}
\usepackage{rotating}
\usepackage{pdflscape}
\usepackage{verbatim}
\usepackage{listings}
\usepackage{xcolor}
\usepackage[shortlabels]{enumitem}
\usepackage{comment}
\usepackage{epstopdf}
\usepackage{algorithm}
\usepackage{algorithmic}

\usepackage{pifont}

\renewcommand{\bm}[1]{\boldsymbol{#1}}
\newcommand{\dd}{\mathrm{d}}
\newcommand{\one}{\mathbf{1}}
\newcommand{\E}{\mathbb{E}}
\newcommand{\diag}{\operatorname{diag}}

\newcommand{\cmark}{\ding{51}}%
\usepackage{natbib}
\bibpunct[, ]{(}{)}{,}{a}{}{,}
\def\bibfont{\small}

\EquationsNumberedThrough    

\TheoremsNumberedThrough     
\ECRepeatTheorems

\MANUSCRIPTNO{}

\begin{document}


\RUNAUTHOR{Jia and Ouyang}

\RUNTITLE{Convergence and Regret of the Policy Gradient for MAB in Diffusion Environment}

\TITLE{Convergence and Regret of the Policy Gradient for Multi-Armed Bandits in Diffusion Environment}

\ARTICLEAUTHORS{%
\AUTHOR{Yanwei Jia}
\AFF{Department of Systems Engineering and Engineering Management, The Chinese University of Hong Kong, Shatin, New Territories, Hong Kong SAR, \EMAIL{yanweijia@cuhk.edu.hk}}

\AUTHOR{Du Ouyang}
\AFF{Department of Mathematical Sciences, Tsinghua University, Beijing, Beijing 100084, China, \EMAIL{duouyang99@outlook.com}}
} 

\ABSTRACT{%
This paper studies the policy gradient update for a multi-arm bandit problem in diffusion environment that is described by a stochastic differential equation (SDE) under the continuous-time reinforcement learning framework by \citet{wang2020reinforcement,jia2022policypgac}. With the logit parameterization for the stochastic policy, we show that it converges almost surely to the optimal arm under an arbitrary constant learning rate. Furthermore, we derive the non-asymptotic regret upper bound when the constant learning rate is below a time-invariant threshold; and the regret bound has order $O(\log T)$. We improve the analysis in \citet{lattimore2026diffusion} for the same SDE by constructing a novel Lyapunov function and demonstrate the transparency of analyzing policy gradient using the tools in SDEs. In addition, the same Lyapunov function is also helpful in analyzing the discrete-time policy gradient algorithm.
}%

\maketitle

\section{Introduction}
\label{sec:intro}

The bandit problem (cf. \citealt{lattimore2020bandit}) is a classical model to describe a decision maker who sequentially selects one of multiple ``arms", collecting the associated random reward, and aims to maximize the expected reward. 
Predominantly, two classes of algorithms based on statistical principles have been extensively studied: upper confidence bound algorithm and Thompson sampling. 
By contrast, little attention has been paid to the gradient method, a generally applicable optimization principle, in bandit problems. Until recently, there is growing theoretical interest in understanding the convergence and regret of policy gradient algorithm in bandit problems, see, e.g., \citet{walton2023regret,mei2023stochastic,mei2024small,baudry2025does,lattimore2026diffusion,lattimore2026lyapunov}.

In this paper, we examine the policy gradient method for the multi-armed bandit (MAB) problem in a diffusion environment. Intuitively speaking, this is the most difficult situation for learning because incrementally, the noise is much larger than the signal (referred to as the ``weak signal regime" in \citealt{kuang2024weak}). We work under the continuous-time reinforcement framework by \citet{wang2020reinforcement}. With a commonly used logit parameterization for the choice probability of arm $a\in \mathcal A = \{1,\cdots,d\}$ given by 
\[
\pi^{(a)}(\bm\phi)
=
\frac{\exp(\phi^{(a)})}
{\sum_{j=1}^d\exp(\phi^{(j)})}, \text{ with }
\bm\pi(\bm\phi)
=
(\pi^{(1)}(\bm\phi),\ldots,\pi^{(d)}(\bm\phi))^\top\in \mathcal P^d,\ \bm\phi=(\phi^{(1)},\ldots,\phi^{(d)})^\top\in \mathbb R^d ,\]
where $ \mathcal P^d
:=
\left\{
\bm p=(p^{(1)},\ldots,p^{(d)})^\top\in[0,1]^d:
\sum_{a=1}^d p^{(a)}=1
\right\}$. \citet{jia2022policypgac} suggest the resulting online, incremental (actor-critic) policy gradient update can be informally described by (see E-Companion \ref{sec:problem-formulation} for the introduction) 
\begin{equation}
\begin{aligned}
\dd\beta_t=&
\alpha_t\bigl(\dd R_t^{(A_t)}-\beta_t\,\dd t\bigr), \\
\dd\bm\phi_t
= &
\ell_t
\nabla_{\bm\phi}\log\pi^{(A_t)}(\bm\phi_t)
\bigl(\dd R_t^{(A_t)}-\beta_t\,\dd t\bigr) = \ell_t \sum_{a=1}^d \one_{\{ A_t = a\}} (\bm e_{a} - \bm\pi(\bm\phi_t))  (\dd R^{(a)}_t - \beta_t\dd t) ,    
\end{aligned}
\label{eq:mab-continuous-actor-critic}
\end{equation}
where $\bm e_a = (0,\cdots,0,1,0,\cdots,0)^\top\in\mathbb R^d$ stands for the unit vector with $a$-th entry 1; $A_t\sim \bm\pi(\bm\phi_t)$ is a random draw; $\dd R^{(A_t)}_t = \mu^{(A_t)}\dd t + \sigma^{(A_t)}\dd B^{(A_t)}_t$ is the selected arm and its associated instantaneous reward; and $\bm\mu = (\mu^{(1)},\cdots,\mu^{(d)})^\top \in \mathbb R^d,\bm\sigma = (\sigma^{(1)},\cdots,\sigma^{(d)})^\top \in \mathbb R^d_{++}$ are the mean and volatility of the reward rate of the each arm. Here, $\alpha_t$ and $\ell_t$ are the critic and actor learning rates respectively. 

With a constant learning rate for the policy $\ell_t\equiv \ell$ and under high-frequency sampling limit, \citet{jia2026accuracy} show that the distribution of $\bm\phi_t$ process can be described by the following well-posed stochastic differential equation (SDE):\footnote{SDE \eqref{eq:prelim-aggregated-sde} can also be viewed as the diffusion limit of the discrete-time policy gradient algorithm in the sense of \citet{fan2021diffusion,kuang2024weak}, see E-Companion \ref{subsec:mab-discrete-implementation}.}
\begin{equation}
\dd\bm\phi_t
=
\ell\bm J(\bm \pi_t)\bm\mu\,\dd t
+
\ell\bm G_{\bm\sigma}(\bm\pi_t)\,\dd\bm B_t^{\phi},
\qquad
\bm\pi_t=\bm\pi(\bm\phi_t),
\label{eq:prelim-aggregated-sde}
\end{equation}
where $\bm B^\phi$ is a standard $d$-dimensional Brownian motion, and $  \bm J(\bm\pi_t)
=
\diag\{\bm\pi_t\}-\bm\pi_t\bm\pi_t^\top\in \mathbb S^{d}_+$, and $
\bm G_{\bm\sigma}(\bm\pi)
= (\bm I-\bm\pi\bm e^\top)
\diag\left\{
\sigma^{(1)}\sqrt{\pi^{(1)}},\ldots,
\sigma^{(d)}\sqrt{\pi^{(d)}}
\right\} \in \mathbb R^{d\times d}$, $ \bm e : = (1,\dots,1)^{\top}\in \mathbb R^d$ 

To understand the behavior of the policy gradient update, it suffices to examine the property of SDE \eqref{eq:prelim-aggregated-sde}. In this paper, we prove that, the policy gradient update \eqref{eq:prelim-aggregated-sde} almost surely converges to the best arm under any arbitrary constant learning rate, and achieved $O(\log T)$ order (instance-dependent) regret upper bound when the learning rate is below a threshold that is determined by the gap in means of each arm and their volatilities. Under this condition, our regret upper bound holds for any finite time $T$. The key analytical tool is to construct a suitable Lyapunov function that induces stabilizing behavior of the underlying process. With the help of It\^o's calculus, such verification becomes much easier. Furthermore, it turns out the same Lyapunov function is also helpful in analyzing the conventional discrete-time policy gradient algorithm.  

The problem we attacked in this paper has recently been studied in \citet{lattimore2026diffusion}, who obtained the similar regret upper bound and a regret lower bound (both of $O(\log T)$ order) by analyzing the same SDE using a different method. We relax the conditions on the learning rate in \citet{lattimore2026diffusion} and give a straightforward, unifying proof for both two-arm and multi-arm cases, and it is valid in both continuous and discrete time environment. Our method is also related to other recent analysis of policy gradient for MAB in the discrete time, such as \citet{walton2023regret,mei2023stochastic,mei2024small,baudry2025does,lattimore2026lyapunov}. In particular, our results on the threshold of the learning rate to ensure logarithmic regret coincide with the conjecture on the maximum learning rate proposed in \citet{baudry2025does}, which scales in $O(\frac{1}{d})$, where $d$ is the number of arms. Table \ref{tab:intro-pg-comparison} gives a comparison of the main results between this paper and these literature.

\begin{table}[h]
\centering
\caption{\textbf{Comparison of results on the policy-gradient for stochastic multi-armed bandits.} The column ``Learning rate" summarizes the condition on the learning rate $ \ell$, where $ d $ is the number of arms, $ T$ and $ N$ are the time horizon for continuous-time and discrete-time systems, respectively. The columns ``a.s. conv." and ``Logarithmic regret" stands for the conclusion of the almost sure convergence and logarithmic expected regret, where ``N/A" means no conclusions. }
\label{tab:intro-pg-comparison}

\scriptsize
\setlength{\tabcolsep}{3pt}
\renewcommand{\arraystretch}{1.02}

\begin{tabular}{@{}
>{\raggedright\arraybackslash}p{0.18\textwidth}
>{\raggedright\arraybackslash}p{0.18\textwidth}
>{\raggedright\arraybackslash}p{0.10\textwidth}
>{\raggedright\arraybackslash}p{0.17\textwidth}
>{\raggedright\arraybackslash}p{0.29\textwidth}
@{}}

\toprule
Literature
& Learning rate
& a.s.\ conv.
& Logarithmic regret 
& Remark
\\
\midrule

\multicolumn{4}{@{}l}{\textit{Panel A: Continuous time}}
\\

\citet{lattimore2026diffusion}
& $\ell \le O(1/\log T)$
& N/A
& \cmark
& Non-asymptotic
\\

\addlinespace[0.45ex]

\textbf{This paper}
& $ \ell \le O(1/d)$
& \cmark
& \cmark
& Non-asymptotic
\\

\addlinespace[0.45ex]

& Any constant $\ell$
& \cmark
& N/A
&
\\

\midrule

\multicolumn{4}{@{}l}{\textit{Panel B: Discrete time}}
\\

\citet{walton2023regret}
& State-dependent
& \cmark
& \cmark
& Non-asymptotic
\\

\addlinespace[0.45ex]

\citet{mei2023stochastic}
& $\ell = O(1/d^{3/2})$
& \cmark
& \cmark
& Nonexplicit constant in regret bound
\\

\addlinespace[0.45ex]

\citet{mei2024small}
& Any constant $\ell$
& \cmark
& N/A
& Pathwise asymptotic $O(\log N)$ 
\\

\addlinespace[0.45ex]

\citet{baudry2025does}
& Small constant $\ell$
& N/A
& \cmark
& Two-arm
\\

\addlinespace[0.45ex]

\citet{lattimore2026lyapunov}
& $\ell \le O(1/\log (dN))$
& N/A
& \cmark
& Non-asymptotic
\\

\addlinespace[0.45ex]

\textbf{This paper}
& $\ell \le O(1/d)$
& \cmark
& \cmark
& Non-asymptotic
\\

\bottomrule

\end{tabular}

\end{table}

In the following, Section \ref{sec:main-results} shows the main results on the regret and almost sure convergence of this SDE. The key construction of the Lyapunov function and the proof is presented. Section \ref{sec:conclusion} makes concluding remarks. E-Companion contains the background of SDE \eqref{eq:prelim-aggregated-sde} and explains where it comes from. We also provide more intuitive illustrations in E-Companion.

\section{Main Results}
\label{sec:main-results}
Unlike \citet{lattimore2026diffusion} who analyzes the two-arm and multi-arm cases separately, we give a unifying account. The special case of two-arm is illustrated in E-Companion \ref{sec:two arm} to highlight several intuitions of the SDE \eqref{eq:prelim-aggregated-sde} and the role of learning rate $\ell$.

\subsection{Logarithmic regret under small learning rate}
\label{subsubsec:mab-heterogeneous-volatilities}

We first establish a finite-time logarithmic expected-regret bound based on SDE \eqref{eq:prelim-aggregated-sde}.  Assume
throughout this subsection that arm $1$ is the unique optimal arm, and put $ \Delta_i=\mu^{(1)}-\mu^{(i)}>0,\
\Delta_{\min}=\min_{2\leq i\leq d}\Delta_i,\
\Delta_{\max}=\max_{2\leq i\leq d}\Delta_i.$
Only the optimal arm is required to be unique; the suboptimal arms may have
equal means.  Define the instantaneous regret rate and its expected cumulative
version by $ r_t
=
\mu^{(1)}-\bm\mu^\top\bm\pi_t
=
\sum_{i=2}^d\Delta_i\pi_t^{(i)}$, and $\mathcal R_T
=
\E\left[\int_0^T r_t\,\dd t\right]$, respectively.
Denote $ s_a:=\sigma^{(a)^2},
\
\overline s:=\max_{1\leq a\leq d}s_a,\
s_{ab}:=s_a\vee s_b,\ 1\leq a<b\leq d.$ Furthermore, denote
$
\bm Q_{\bm\sigma}(\bm\pi) =\bm G_{\bm\sigma}\bm G_{\bm\sigma}^\top
= 
(\bm I-\bm\pi\bm e^\top)
\diag\{\pi^{(1)} \sigma^{(1)^2},\ldots,\pi^{(d)}\sigma^{(d)^2}\}
(\bm I-\bm\pi\bm e^\top)^\top \in \mathbb S^d_+$.

\begin{lemma}
\label{lem:mab-wellposedness}
For every finite deterministic initial condition $\bm\phi_0$, the SDE
\eqref{eq:prelim-aggregated-sde} has a unique nonexplosive strong solution.
Moreover, $\bm e^\top\bm\phi_t=\bm e^\top\bm\phi_0$ for all $t\geq 0$ almost surely.
\end{lemma}

\begin{proof}{Proof.}
For every finite $\bm\phi$, all softmax probabilities are strictly positive.
The drift and volatility of SDE \eqref{eq:prelim-aggregated-sde} are global Lipschitz functions of
$\bm\phi$.  They are also bounded because $\bm\pi$ takes values in the
probability simplex and $ \|\bm G_{\bm\sigma}(\bm\pi)\|_{\mathrm F}^2
=
\operatorname{tr}\bigl(\bm Q_{\bm\sigma}(\bm\pi)\bigr)
\leq
\overline s\sum_{a=1}^d
\pi^{(a)}\|\bm e_a-\bm\pi\|_2^2
=
\overline s\,\operatorname{tr}(\bm J)
\leq\overline s.$
Standard SDE theory (e.g., \citealt[Chapter 5, Theorem 2.5]{karatzas1991brownian}) therefore gives a unique global strong solution.
Furthermore, by direction calculation, $\bm e^\top\bm J=\bm0^\top$ and
$\bm e^\top\bm G_{\bm\sigma}=\bm0^\top$.  Multiplying
\eqref{eq:prelim-aggregated-sde} by $\bm e^\top$ and integrating from 0 to $t$ gives
the desired result. \Halmos
\end{proof}

Next, we shall repeatedly use the elementary identity
\begin{equation}
\bm v^\top\bm J(\bm\pi)\bm w
=
\sum_{1\leq a<b\leq d}
\pi^{(a)}\pi^{(b)}
(v_a-v_b)(w_a-w_b),
\qquad
\bm v,\bm w\in\mathbb R^d.
\label{eq:mab-pairwise-J-identity}
\end{equation}
To see this, expanding the right-hand side as one half of the corresponding double
sum gives
\[
\sum_{a=1}^d\pi^{(a)}v_aw_a
-
\left(\sum_{a=1}^d\pi^{(a)}v_a\right)
\left(\sum_{b=1}^d\pi^{(b)}w_b\right)
=\bm v^\top\bm J(\bm\pi)\bm w.
\]

Let $I=\{2,\ldots,d\}$ be the set of suboptimal arm index. For every nonempty $A\subseteq I$, define
\begin{equation}
h_A(\bm\phi)
=
\prod_{i\in A}e^{\phi^{(i)}-\phi^{(1)}}
=
\prod_{i\in A}\frac{\pi^{(i)}}{\pi^{(1)}}.
\label{eq:mab-subset-monomial}
\end{equation}

We first record an estimate on the covariance matrix useful for verifying Lyapunov function.
\begin{lemma}
\label{lem:mab-heterogeneous-pairwise-covariance}
For every probability vector $\bm\pi\in \mathcal P^d$ and every $\bm v\in\mathbb R^d$,
\begin{equation}
\bm v^\top\bm Q_{\bm\sigma}(\bm\pi)\bm v
\le
\sum_{1\le a<b\le d}
\pi^{(a)}\pi^{(b)}s_{ab}(v_a-v_b)^2.
\label{eq:mab-heterogeneous-pairwise-covariance-bound}
\end{equation}
In particular,
\begin{equation}
\bm Q_{\bm\sigma}(\bm\pi)
\preceq \overline s\,\bm J(\bm\pi).
\label{eq:mab-heterogeneous-Loewner-bound}
\end{equation}
\end{lemma}

\begin{proof}{Proof}
Write $\overline v=\bm\pi^\top\bm v$. Then $ \bm v^\top\bm Q_{\bm\sigma}(\bm\pi)\bm v
=\sum_{a=1}^d\pi^{(a)}s_a(v_a-\overline v)^2.$
Fix
$H\subseteq\{1,\ldots,d\}$ and let $p=\sum_{a\in H}\pi^{(a)}$.  Suppose that $0<p<1$. Denote
\begin{equation*}
\begin{aligned}
m_H
&:=\frac{1}{p}\sum_{a\in H}\pi^{(a)}v_a,
&
V_H
&:=\frac{1}{p}\sum_{a\in H}\pi^{(a)}(v_a-m_H)^2,\\
m_{H^c}
&:=\frac{1}{1-p}\sum_{a\in H^c}\pi^{(a)}v_a,
&
V_{H^c}
&:=\frac{1}{1-p}\sum_{a\in H^c}
\pi^{(a)}(v_a-m_{H^c})^2.
\end{aligned}
\end{equation*}
Since
$\overline v=pm_H+(1-p)m_{H^c}$,
\begin{equation}
\sum_{a\in H}\pi^{(a)}(v_a-\overline v)^2
=pV_H+p(1-p)^2(m_H-m_{H^c})^2.
\label{eq:mab-heterogeneous-subset-left}
\end{equation}
On the other hand, separating the pairs inside $H$ from those crossing from
$H$ to $H^c$ gives
\begin{equation}
\begin{aligned}
&\sum_{a<b}\pi^{(a)}\pi^{(b)}
\one_{\{a\in H\ \mathrm{or}\ b\in H\}}(v_a-v_b)^2\\
&\quad
=pV_H+p(1-p)V_{H^c}
+p(1-p)(m_H-m_{H^c})^2.
\end{aligned}
\label{eq:mab-heterogeneous-subset-right}
\end{equation}
The right-hand side of
\eqref{eq:mab-heterogeneous-subset-right} minus
\eqref{eq:mab-heterogeneous-subset-left} equals $ p(1-p)V_{H^c}
+p^2(1-p)(m_H-m_{H^c})^2\geq0$. Thus,
\begin{equation}\label{eq:subset-inequality}
   \sum_{a\in H}\pi^{(a)}(v_a-\overline v)^2 \le \sum_{a<b}\pi^{(a)}\pi^{(b)}
\one_{\{a\in H\ \mathrm{or}\ b\in H\}}(v_a-v_b)^2 . 
\end{equation}
The above inequality also holds when $ p \in \{0,1\}$. 

For every $x\geq0$, the indicator
$\one_{\{x\geq u\}}$ equals one for $u\in[0,x]$ and zero for $u>x$.
Consequently, since $s_{ab}=s_a\vee s_b$, we can write
\begin{equation}
s_a=\int_0^\infty\one_{\{s_a\geq u\}}\,\dd u,
\qquad
s_{ab}=\int_0^\infty
\one_{\{s_a\geq u\ \mathrm{or}\ s_b\geq u\}}\,\dd u.
\label{eq:mab-heterogeneous-layer-cake}
\end{equation}
For each $u\geq0$, set
$H_u:=\{a:s_a\geq u\}$.  Applying the inequality \eqref{eq:subset-inequality}
with $H=H_u$ gives
\begin{equation}
\begin{aligned}
\sum_{a=1}^d
\pi^{(a)}\one_{\{s_a\geq u\}}(v_a-\overline v)^2\leq
\sum_{a<b}\pi^{(a)}\pi^{(b)}
\one_{\{s_a\geq u\ \mathrm{or}\ s_b\geq u\}}
(v_a-v_b)^2.
\end{aligned}
\label{eq:mab-heterogeneous-level-set-inequality}
\end{equation}
All terms in \eqref{eq:mab-heterogeneous-level-set-inequality} are
nonnegative, and both sums are finite.  We may therefore integrate over
$u\in[0,\infty)$ and interchange each sum with the integral.  By
\eqref{eq:mab-heterogeneous-layer-cake}, the integral of the left-hand side is
\begin{equation}
\begin{aligned}
\int_0^\infty\sum_{a=1}^d
\pi^{(a)}\one_{\{s_a\geq u\}}
(v_a-\overline v)^2\,\dd u=
\sum_{a=1}^d\pi^{(a)}s_a(v_a-\overline v)^2
=\bm v^\top\bm Q_{\bm\sigma}(\bm\pi)\bm v.
\end{aligned}
\label{eq:mab-heterogeneous-layer-cake-left}
\end{equation}
Similarly, the
integral of the right-hand side is
\begin{equation}
\begin{aligned}
\int_0^\infty\sum_{a<b}\pi^{(a)}\pi^{(b)}
\one_{\{s_a\geq u\ \mathrm{or}\ s_b\geq u\}}
(v_a-v_b)^2\,\dd u=
\sum_{a<b}\pi^{(a)}\pi^{(b)}
s_{ab}(v_a-v_b)^2.
\end{aligned}
\label{eq:mab-heterogeneous-layer-cake-right}
\end{equation}
Integrating \eqref{eq:mab-heterogeneous-level-set-inequality} and using
\eqref{eq:mab-heterogeneous-layer-cake-left}--\eqref{eq:mab-heterogeneous-layer-cake-right}
proves \eqref{eq:mab-heterogeneous-pairwise-covariance-bound}.  Finally,
$s_{ab}\leq\overline s$ and
\eqref{eq:mab-pairwise-J-identity} imply, for every $\bm v\in\mathbb R^d$,
\begin{equation*}
\bm v^\top\bm Q_{\bm\sigma}(\bm\pi)\bm v
\leq
\overline s\sum_{a<b}\pi^{(a)}\pi^{(b)}(v_a-v_b)^2
=\overline s\,\bm v^\top\bm J(\bm\pi)\bm v.
\label{eq:mab-heterogeneous-Loewner-quadratic-form}
\end{equation*}
This proves \eqref{eq:mab-heterogeneous-Loewner-bound}. \Halmos
\end{proof}

\begin{theorem}
\label{thm:mab-heterogeneous-volatility-regret}
Let $d\geq2$ and let $\bm\phi_0$ be finite and deterministic.  Assume that arm
$1$ is uniquely optimal.  Suppose $ 0<\ell<
\min_{\substack{2\leq j\leq d}}
\frac{2\Delta_j}{d s_{1j}}.$
Define $ c_{\ell,\bm\sigma}
:=
\min_{2\le j\le d}
\left\{
\Delta_j-\frac{\ell d}{2}s_{1j}
\right\} >0$.
Denote
\begin{equation}
K_{\ell,\bm\sigma}
=\Delta_{\max}+\frac{\ell\overline s}{2},
\quad
C_{\ell,\bm\sigma}
=\frac{K_{\ell,\bm\sigma}}{c_{\ell,\bm\sigma}},\quad S_0=\bm e^\top\bm\phi_0,
\quad
a_{\ell,\bm\sigma}
=\frac{e^{-S_0/d}}{d}
\left(
\ell\Delta_{\max}+\frac{\ell^2\overline s}{2}
\right),
\label{eq:mab-heterogeneous-KC}
\end{equation}
and define the Lyapunov function
\begin{equation}
\begin{aligned}
V_{\ell,\bm\sigma}(\bm\phi)
&:=\sum_{\varnothing\neq A\subseteq I}
C_{\ell,\bm\sigma}^{|A|-1}h_A(\bm\phi)=\frac{1}{C_{\ell,\bm\sigma}}
\left(
\prod_{i=2}^d
\left(
1+C_{\ell,\bm\sigma}e^{\phi^{(i)}-\phi^{(1)}}
\right)-1
\right).
\end{aligned}
\label{eq:mab-heterogeneous-subset-potential}
\end{equation}
For the same fixed learning rate, for every $T\geq0$,
\begin{equation}
\begin{aligned}
\mathcal R_T
\le
\frac1\ell\sum_{i=2}^d
\log\left(
1+a_{\ell,\bm\sigma}e^{\phi_0^{(i)}}T
\right)+
\frac{\Delta_{\max}(d-1)}
{2\ell c_{\ell,\bm\sigma}}
V_{\ell,\bm\sigma}(\bm\phi_0).
\end{aligned}
\label{eq:mab-heterogeneous-regret-general-initial}
\end{equation}
In particular, if $\bm\phi_0=\bm0$, then
\begin{equation}
\begin{aligned}
\mathcal R_T
\le
\frac{d-1}{\ell}
\log\left(
1+
\frac{\ell\Delta_{\max}+\frac12\ell^2\overline s}{d}T
\right)+
\frac{\Delta_{\max}(d-1)}
{2\ell c_{\ell,\bm\sigma}}
\frac{(1+C_{\ell,\bm\sigma})^{d-1}-1}
{C_{\ell,\bm\sigma}}.
\end{aligned}
\label{eq:mab-heterogeneous-regret-zero-initial}
\end{equation}
\end{theorem}

\begin{proof}{Proof}
Let $\mathcal L_{\bm\sigma}$ be the generator of
\eqref{eq:prelim-aggregated-sde}.  For
$f_{\bm v}(\bm\phi)=\exp(\bm v^\top\bm\phi)$,
Lemma~\ref{lem:mab-heterogeneous-pairwise-covariance} and
\eqref{eq:mab-pairwise-J-identity} imply
\begin{equation}
\begin{aligned}
\frac{\mathcal L_{\bm\sigma}f_{\bm v}(\bm \phi)}{f_{\bm v}(\bm \phi)}
&=\ell\bm v^\top\bm J(\bm \pi)\bm\mu
+\frac{\ell^2}{2}
\bm v^\top\bm Q_{\bm\sigma}\bm v\\
&\le
\sum_{a<b}\pi^{(a)}\pi^{(b)}
\left[
\ell(v_a-v_b)(\mu^{(a)}-\mu^{(b)})
+\frac{\ell^2}{2}s_{ab}(v_a-v_b)^2
\right].
\end{aligned}
\label{eq:mab-heterogeneous-exponential-generator-bound}
\end{equation}

Fix a nonempty $A\subseteq I$ and write $m=|A|\leq d-1$.  Define
$\bm v^A\in\mathbb R^d$ by $ v_1^A=-m,
\
v_i^A=\one_{\{i\in A\}},\ 2\leq i\leq d.$
By \eqref{eq:mab-subset-monomial},
$h_A(\bm\phi)=\exp((\bm v^A)^\top\bm\phi)$. We examine the square bracket term in \eqref{eq:mab-heterogeneous-exponential-generator-bound}. The pair $ (1,j)$ with $j\in A$ has coefficient
\begin{equation}
\begin{aligned}
&-\ell(m+1)\Delta_j
+\frac{\ell^2}{2}(m+1)^2s_{1j}
=-\ell(m+1)
\left[
\Delta_j-\frac{\ell(m+1)}{2}s_{1j}
\right]
\le-\ell(m+1)c_{\ell,\bm\sigma},
\end{aligned}
\label{eq:mab-heterogeneous-best-pair-selected}
\end{equation}
where $m+1\le d$ was used.  If $j\in I\setminus A$, the analogous
coefficient is $ -\ell m
\left(
\Delta_j-\frac{\ell m}{2}s_{1j}
\right)
\le-\ell m c_{\ell,\bm\sigma}.$
For a pair of suboptimal arms, the contribution vanishes when both indices
belong to $A$ or both lie outside $A$.  If exactly one belongs to $A$, its
coefficient is at most $ \ell|\mu^{(i)}-\mu^{(j)}|
+\frac{\ell^2}{2}s_{ij}
\le\ell K_{\ell,\bm\sigma}.$
Consequently,
\begin{equation}
\begin{aligned}
\mathcal L_{\bm\sigma}h_A
\le{}&
-\ell c_{\ell,\bm\sigma}h_A\pi^{(1)}
\left[
(m+1)\sum_{i\in A}\pi^{(i)}
+m\sum_{j\in I\setminus A}\pi^{(j)}
\right]+\ell K_{\ell,\bm\sigma}h_A
\sum_{\substack{i\in A\\j\in I\setminus A}}
\pi^{(i)}\pi^{(j)}.
\end{aligned}
\label{eq:mab-heterogeneous-subset-generator-upper}
\end{equation}

To simplify notation within this
calculation, write $ c=c_{\ell,\bm\sigma},
\ 
K=K_{\ell,\bm\sigma},
\ 
C=C_{\ell,\bm\sigma}.$
We next consider the upper bound $ \mathcal{L}_\sigma V_{\ell,\sigma}\le P + N $, where, due to \eqref{eq:mab-heterogeneous-subset-generator-upper}, the possibly positive contribution is
\begin{equation}
P
:=
\ell K
\sum_{\varnothing\neq A\subseteq I}
C^{|A|-1}h_A
\sum_{\substack{i\in A\\j\in I\setminus A}}
\pi^{(i)}\pi^{(j)},
\label{eq:mab-heterogeneous-positive-contribution}
\end{equation}
and the negative contribution (with the negative term $m\sum_{j\in I\setminus A}\pi^{(j)}$ in \eqref{eq:mab-heterogeneous-subset-generator-upper} being relaxed to 0) is
\begin{equation}
N
:=
-\ell c
\sum_{\varnothing\neq A\subseteq I}
(|A|+1)C^{|A|-1}h_A\pi^{(1)}
\sum_{i\in A}\pi^{(i)}.
\label{eq:mab-heterogeneous-negative-contribution}
\end{equation}

For a triple $(A,i,j)$ occurring in
\eqref{eq:mab-heterogeneous-positive-contribution}, set
$D=A\cup\{j\}$ and $r=|D|=|A|+1$.  Since $ h_D=h_Ae^{\phi^{(j)}-\phi^{(1)}}
=h_A\frac{\pi^{(j)}}{\pi^{(1)}},$
we have the identity
\begin{equation}
h_A\pi^{(i)}\pi^{(j)}
=
h_D\pi^{(1)}\pi^{(i)}.
\label{eq:mab-heterogeneous-parent-child-identity}
\end{equation}
Conversely, fix $D\subseteq I$ with $|D|=r\geq2$ and fix $i\in D$.  Therefore,
\begin{equation}
P
=
\ell K\sum_{r=2}^{d-1}(r-1)C^{r-2}
\sum_{\substack{D\subseteq I\\|D|=r}}
h_D\pi^{(1)}\sum_{i\in D}\pi^{(i)}.
\label{eq:mab-heterogeneous-positive-reindexed}
\end{equation}
Grouping \eqref{eq:mab-heterogeneous-negative-contribution} by $r=|D|$
similarly gives
\begin{equation}
N
=
-\ell c\sum_{r=1}^{d-1}(r+1)C^{r-1}
\sum_{\substack{D\subseteq I\\|D|=r}}
h_D\pi^{(1)}\sum_{i\in D}\pi^{(i)}.
\label{eq:mab-heterogeneous-negative-reindexed}
\end{equation}
For every $r\geq2$, the combined coefficient in
\eqref{eq:mab-heterogeneous-positive-reindexed} and
\eqref{eq:mab-heterogeneous-negative-reindexed} is
\begin{equation}
\begin{aligned}
\ell C^{r-2}
\left[
K(r-1)-cC(r+1)
\right]
&=
\ell K C^{r-2}\bigl[(r-1)-(r+1)\bigr]=-2\ell K C^{r-2}\leq0,
\end{aligned}
\label{eq:mab-heterogeneous-levelwise-cancellation}
\end{equation}
where $cC=K$ by the provided condition \eqref{eq:mab-heterogeneous-KC}.  At level $r=1$,
there is no positive term.  Moreover,
\[
h_{\{i\}}\pi^{(1)}\pi^{(i)}
=
\frac{\pi^{(i)}}{\pi^{(1)}}\pi^{(1)}\pi^{(i)}
=
(\pi^{(i)})^2.
\]
Thus, the singleton level contributes exactly
$-2\ell c\sum_{i=2}^d(\pi^{(i)})^2$, and every higher level is nonpositive by
\eqref{eq:mab-heterogeneous-levelwise-cancellation}.  We have therefore
proved the global Lyapunov inequality
\begin{equation}
\mathcal L_{\bm\sigma}V_{\ell,\bm\sigma}(\bm\phi)
\le
-2\ell c_{\ell,\bm\sigma}
\sum_{i=2}^d(\pi^{(i)})^2.
\label{eq:mab-heterogeneous-LV-bound}
\end{equation}

To the desired finite-time regret bound, denote $ \tau_n:=\inf\{t\geq0:\|\bm\phi_t\|_2\geq n\}.$
By Lemma~\ref{lem:mab-wellposedness}, $\tau_n\uparrow\infty$ almost surely.
On $[0,T\wedge\tau_n]$, the state remains in a compact set, and the stochastic
integral in It\^{o}'s formula for $V_{\ell,\bm\sigma}$ is therefore a true
martingale.  Hence \eqref{eq:mab-heterogeneous-LV-bound} gives
\[
\begin{aligned}
0\leq \E\left[
V_{\ell,\bm\sigma}(\bm\phi_{T\wedge\tau_n})
\right]
&=
V_{\ell,\bm\sigma}(\bm\phi_0)
+
\E\int_0^{T\wedge\tau_n}
\mathcal L_{\bm\sigma}
V_{\ell,\bm\sigma}(\bm\phi_t)\,\dd t\\
&\leq
V_{\ell,\bm\sigma}(\bm\phi_0)
-
2\ell c_{\ell,\bm\sigma}
\E\int_0^{T\wedge\tau_n}
\sum_{i=2}^d(\pi_t^{(i)})^2\,\dd t.
\end{aligned}
\]
Applying the monotone convergence theorem to the right-hand side, and first letting
$n\to\infty$, and then letting $T\to\infty$; this proves
\begin{equation}
\E\left[
\int_0^\infty\sum_{i=2}^d(\pi_t^{(i)})^2\,\dd t
\right]
\le
\frac{V_{\ell,\bm\sigma}(\bm\phi_0)}
{2\ell c_{\ell,\bm\sigma}}.
\label{eq:mab-heterogeneous-occupation-bound}
\end{equation}

It remains to control the part of regret containing the best-arm probability.
Fix $i\in I$ and let $ Y_i(\bm\phi)=e^{-\phi^{(i)}}$.
The gradient and Hessian of $Y_i$ are
$-Y_i\bm e_i$ and $Y_i\bm e_i\bm e_i^\top$, respectively.  Moreover,
\eqref{eq:mab-heterogeneous-Loewner-bound} gives $ (\bm Q_{\bm\sigma})_{ii}
\leq
\overline s J_{ii}
=
\overline s\,\pi^{(i)}(1-\pi^{(i)}).$
It\^{o}'s formula therefore gives
\begin{equation}
\begin{aligned}
\mathcal L_{\bm\sigma}Y_i
&=Y_i
\left[
\ell\pi^{(i)}
(\bm\mu^\top\bm\pi-\mu^{(i)})
+\frac{\ell^2}{2}(\bm Q_{\bm\sigma})_{ii}
\right]\le
Y_i\pi^{(i)}
\left(
\ell\Delta_{\max}+\frac{\ell^2\overline s}{2}
\right),
\end{aligned}
\label{eq:mab-heterogeneous-negative-logit-generator}
\end{equation}
where
$\bm\mu^\top\bm\pi-\mu^{(i)}
\leq\mu^{(1)}-\mu^{(i)}
=\Delta_i\leq\Delta_{\max}$.
Writing $ Z(\bm\phi)=\sum_{j=1}^d e^{\phi^{(j)}},$
the softmax definition gives $Y_i\pi^{(i)}=1/Z(\bm\phi)$.  By the
Jensen's inequality and
Lemma \ref{lem:mab-wellposedness}, $ Z(\bm\phi_t)
\geq
d\left(\prod_{j=1}^de^{\phi_t^{(j)}}\right)^{1/d}
=
d e^{S_0/d}.$
Combining this inequality with
\eqref{eq:mab-heterogeneous-negative-logit-generator} gives the pointwise bound $ \mathcal L_{\bm\sigma}Y_i(\bm\phi_t)
\leq a_{\ell,\bm\sigma}.$

By applying It\^{o}'s formula on $[0,t\wedge\tau_n]$, we obtain $ \E Y_i(\bm\phi_{t\wedge\tau_n})
\leq
e^{-\phi_0^{(i)}}+a_{\ell,\bm\sigma}t.$
Since $\tau_n\uparrow\infty$ almost surely and $Y_i\geq0$, by letting $n\to\infty$, Fatou's lemma gives
\begin{equation}
\E e^{-\phi_t^{(i)}}
\le e^{-\phi_0^{(i)}}+a_{\ell,\bm\sigma}t.
\label{eq:mab-heterogeneous-exponential-moment}
\end{equation}
The coefficients of \eqref{eq:prelim-aggregated-sde} are bounded, so every
coordinate of $\bm\phi_t$ is integrable at any finite time. Moreover, Lemma \ref{lem:mab-wellposedness} 
gives $ \phi_T^{(1)}-\phi_0^{(1)}
=
\sum_{i=2}^d
\left(\phi_0^{(i)}-\phi_T^{(i)}\right).$
By Jensen's inequality and
\eqref{eq:mab-heterogeneous-exponential-moment},
\begin{equation}
\begin{aligned}
\E[\phi_T^{(1)}-\phi_0^{(1)}]
&=
\sum_{i=2}^d
\E\log\left(
e^{\phi_0^{(i)}}e^{-\phi_T^{(i)}}
\right)\leq
\sum_{i=2}^d
\log\left(
e^{\phi_0^{(i)}}\E e^{-\phi_T^{(i)}}
\right)\leq
\sum_{i=2}^d
\log\left(
1+a_{\ell,\bm\sigma}e^{\phi_0^{(i)}}T
\right).
\end{aligned}
\label{eq:mab-heterogeneous-best-logit-bound}
\end{equation}

Finally, the first coordinate of
\eqref{eq:prelim-aggregated-sde} satisfies $ \dd\phi_t^{(1)}
=\ell\pi_t^{(1)}r_t\,\dd t
+\ell\bm e_1^\top
\bm G_{\bm\sigma}(\bm\pi_t)\,\dd\bm B_t^{\phi}.$
The stochastic integrand is bounded, so its integral is square-integrable on
every finite interval.  Consequently,
\begin{equation}
\E[\phi_T^{(1)}-\phi_0^{(1)}]
=\ell\E\int_0^T\pi_t^{(1)}r_t\,\dd t.
\label{eq:mab-heterogeneous-best-logit-expectation}
\end{equation}
Denote $ u_t:=1-\pi_t^{(1)}=\sum_{i=2}^d\pi_t^{(i)}.$
Since $r_t\leq\Delta_{\max}u_t$, the Cauchy--Schwarz inequality gives $ u_tr_t
\leq
\Delta_{\max}u_t^2
\leq
\Delta_{\max}(d-1)
\sum_{i=2}^d(\pi_t^{(i)})^2.$
Using
$r_t=\pi_t^{(1)}r_t+u_tr_t$, and then applying
\eqref{eq:mab-heterogeneous-occupation-bound},
\eqref{eq:mab-heterogeneous-best-logit-bound}, and
\eqref{eq:mab-heterogeneous-best-logit-expectation}, we obtain
\[
\begin{aligned}
\mathcal R_T
=
\E\int_0^T\pi_t^{(1)}r_t\,\dd t
+
\E\int_0^Tu_tr_t\,\dd t&\leq
\frac1\ell
\E[\phi_T^{(1)}-\phi_0^{(1)}]
+
\Delta_{\max}(d-1)
\E\int_0^T
\sum_{i=2}^d(\pi_t^{(i)})^2\,\dd t\\
&\leq
\frac1\ell\sum_{i=2}^d
\log\left(
1+a_{\ell,\bm\sigma}e^{\phi_0^{(i)}}T
\right)
+
\frac{\Delta_{\max}(d-1)}
{2\ell c_{\ell,\bm\sigma}}
V_{\ell,\bm\sigma}(\bm\phi_0).
\end{aligned}
\]
This proves \eqref{eq:mab-heterogeneous-regret-general-initial}.

If $\bm\phi_0=\bm0$, then $S_0=0$ and every $h_A(\bm0)=1$.  By the binomial
theorem, $ V_{\ell,\bm\sigma}(\bm0)
=
\sum_{m=1}^{d-1}
\binom{d-1}{m}C_{\ell,\bm\sigma}^{m-1}
=
\frac{(1+C_{\ell,\bm\sigma})^{d-1}-1}
{C_{\ell,\bm\sigma}}.$
Substitution into
\eqref{eq:mab-heterogeneous-regret-general-initial} proves
\eqref{eq:mab-heterogeneous-regret-zero-initial}. \Halmos
\end{proof}

\subsection{Almost sure convergence under arbitrary constant learning rate}
\label{subsec:mab-heterogeneous-arbitrary-learning-rate}

The logarithmic expected-regret theorem requires
a small learning rate $ \ell$.  Almost sure convergence
has a different threshold: under pairwise distinct $ \mu^{(i)}$, it holds for every
fixed $\ell>0$.  

The following lemma isolates the stochastic comparison principle used in the
arbitrary-learning-rate argument.  Its purpose is to absorb an integrable
positive coefficient in a Lyapunov drift inequality and to retain the
convergence and occupation-time consequences of the remaining negative drift.

\begin{lemma}
\label{lem:mab-integrating-factor-supermartingale}
Let $s\geq0$ be deterministic, let $\tau\geq s$ be a stopping time which may
take the value $\infty$, and let $X$ be a nonnegative continuous semimartingale
such that, for every $t\geq s$,
\begin{equation}
X_{t\wedge\tau}
=X_s+\int_s^{t\wedge\tau}\beta_u\,\dd u
+M_{t\wedge\tau}-M_s,
\label{eq:mab-integrating-factor-semimartingale}
\end{equation}
where $X_s\in L^1$, $M$ is a continuous local martingale, and $\beta$ is
progressively measurable and locally integrable.  Suppose that there are a
constant $c>0$ and nonnegative progressively measurable, locally integrable
processes $\alpha$ and $g$ such that
\begin{equation}
\beta_u\le\alpha_uX_u-cX_ug_u
\quad\text{for }s\le u<\tau,
\qquad \dd u\otimes\dd\mathbb P\text{-almost everywhere}.
\label{eq:mab-integrating-factor-drift-hypothesis}
\end{equation}
Define $ A_t:=\int_s^{t\wedge\tau}\alpha_u\,\dd u,
\
Y_t:=e^{-A_t}X_{t\wedge\tau},
\ t\geq s.$
Then $Y$ is a nonnegative supermartingale and converges almost surely
to a finite limit.  Moreover,
\begin{equation}
\int_s^\tau e^{-A_u}X_ug_u\,\dd u<\infty
\qquad\text{almost surely}.
\label{eq:mab-integrating-factor-weighted-occupation}
\end{equation}
On the event $\{A_\infty<\infty\}$, one further has
\begin{equation}
X_{t\wedge\tau}\ \text{converges to a finite limit as }t\to\infty,
\text{ and }
\int_s^\tau X_ug_u\,\dd u<\infty.
\label{eq:mab-integrating-factor-consequences}
\end{equation}
\end{lemma}

\begin{proof}{Proof.}
Set $q_u:=\alpha_uX_u-\beta_u$. By \eqref{eq:mab-integrating-factor-drift-hypothesis},
$q_u\geq cX_ug_u\geq0$ for $s\le u<\tau$, almost everywhere.  By integration by
parts in \eqref{eq:mab-integrating-factor-semimartingale}, and the fact that
$A$ is continuous and has finite variation, we obtain
\begin{equation}
Y_t
=X_s-K_t+\widetilde M_t,
\qquad
K_t:=\int_s^{t\wedge\tau}e^{-A_u}q_u\,\dd u,
\qquad
\widetilde M_t:=\int_s^{t\wedge\tau}e^{-A_u}\,\dd M_u.
\label{eq:mab-integrating-factor-decomposition}
\end{equation}
The process $K$ is continuous, adapted, and nondecreasing, while
$\widetilde M$ is a continuous local martingale.  Since $Y\geq0$,
\eqref{eq:mab-integrating-factor-decomposition} shows that $Y$ is a
nonnegative local supermartingale, and hence, a supermartingale. By
supermartingale convergence theorem implies that $Y_t$ converges almost surely
to a finite limit (cf. \citealt[Chapter~1, Theorem 3.15]{karatzas1991brownian}).

We next prove the rest of the statement by the localization techniques. For
$n\in\mathbb N$, define $ \eta_n
:=\inf\{t\geq s:|\widetilde M_t|\geq n\}\wedge(s+n),
\
\theta_n
:=\eta_n\wedge\inf\{t\geq s:K_t\geq n\},$
where $\inf\varnothing=\infty$.  Continuity implies that
$\widetilde M_{t\land\eta_n}$ is bounded, and hence is a martingale.  Since
$\theta_n\le\eta_n\le s+n$, optional stopping gives
$\E[\widetilde M_{\theta_n}]=0$.  We consider the stopped process
\eqref{eq:mab-integrating-factor-decomposition} at $\theta_n$, and take the expectation at the deterministic time $s+n$, since
$Y_{\theta_n}\geq0$;and it yields
\begin{equation}
\E[K_{\theta_n}]
=\E[X_s]-\E[Y_{\theta_n}]
\le\E[X_s].
\label{eq:mab-integrating-factor-stopped-bound}
\end{equation}
The stopping times $\theta_n$ increase to $\infty$ almost surely.  Indeed,
$\eta_n\uparrow\infty$ by continuity of $\widetilde M$, and the hitting times
of the levels $n$ by the continuous process $K$, which is finite on every
compact interval, also tend to infinity.  Thus
$K_{\theta_n}\uparrow K_\infty$, and monotone convergence in
\eqref{eq:mab-integrating-factor-stopped-bound} gives
\begin{equation*}
\E[K_\infty]
=\E\left[\int_s^\tau e^{-A_u}q_u\,\dd u\right]
\le\E[X_s]<\infty.
\label{eq:mab-integrating-factor-compensator-bound}
\end{equation*}
In particular, $K_\infty<\infty$ almost surely.  Moreover, since
$q_u\geq cX_ug_u$, the above also implies \eqref{eq:mab-integrating-factor-weighted-occupation}.

Finally, on $\{A_\infty<\infty\}$, both $e^{A_t}$ and $Y_t$ have finite
limits.  Hence $X_{t\wedge\tau}=e^{A_t}Y_t$ has a finite limit.  Moreover,
$A_u\le A_\infty$ and therefore $ \int_s^\tau X_ug_u\,\dd u
\le e^{A_\infty}
\int_s^\tau e^{-A_u}X_ug_u\,\dd u
<\infty.$
This proves \eqref{eq:mab-integrating-factor-consequences}. \Halmos
\end{proof}

\begin{theorem}
\label{thm:mab-heterogeneous-arbitrary-learning-rate-as-convergence}
\label{thm:mab-arbitrary-learning-rate-as-convergence}
Assume there are no ties in the expected reward rate among the arms, and
\begin{equation}
\mu^{(1)}>\mu^{(2)}>\cdots>\mu^{(d)},
\qquad
\rho:=\min_{1\le a<b\le d}
\bigl(\mu^{(a)}-\mu^{(b)}\bigr)>0.
\label{eq:mab-strict-mean-order}
\end{equation}
Then for every fixed constant learning rate $\ell>0$ and every finite
deterministic initial condition $\bm\phi_0$, $
\bm\pi_t\longrightarrow\bm e_1$ almost surely. That is, the policy converges to the best arm.
\end{theorem}

\begin{proof}{Proof.}
Fix an arbitrary $\ell>0$, and set $\gamma=
\displaystyle
\min\left\{1,\frac{\rho}{2\ell\overline s}\right\}$, then it satisfies
\begin{equation}
0<\gamma\le1,
\quad
\gamma\ell\overline s\le\frac{\rho}{2}.
\label{eq:mab-fractional-exponent}
\end{equation}
For $j\in\{2,\ldots,d\}$, define
\begin{equation}
\bm v^{(j)}:=\bm e_j-\bm e_1,
\qquad
H_j(\bm\phi)
:=\exp\bigl(\gamma(\phi^{(j)}-\phi^{(1)})\bigr)
=\left(\frac{\pi^{(j)}}{\pi^{(1)}}\right)^\gamma,
\qquad
\Lambda_j(\bm\pi)
:=\sum_{k=j+1}^d\pi^{(k)},
\label{eq:mab-fractional-ratio-and-lower-mass}
\end{equation}
where the empty sum gives $\Lambda_d(\bm\pi)=0$. Denote
\begin{equation}
a_\gamma:=\frac{\gamma\ell\rho}{2}>0,
\qquad
b_\gamma:=\gamma\ell
\left(\Delta_{\max}+\frac{\gamma\ell\overline s}{2}\right)>0.
\label{eq:mab-fractional-drift-constants}
\end{equation}

The generator of $H_j$ is
\[
\frac{\mathcal L_{\bm\sigma}H_j(\bm \phi)}{H_j(\bm \phi)}
=
\gamma\ell(\bm v^{(j)})^\top\bm J(\bm \pi)\bm\mu
+
\frac{\gamma^2\ell^2}{2}
(\bm v^{(j)})^\top\bm Q_{\bm\sigma}(\bm \pi)\bm v^{(j)}.
\]
We now apply Lemma~\ref{lem:mab-heterogeneous-pairwise-covariance} to the
quadratic term and use \eqref{eq:mab-pairwise-J-identity} for the drift.
Separating all pairs which contain index $1$ or index $j$ gives the upper bound
\begin{equation}
\begin{aligned}
\frac{\mathcal L_{\bm\sigma}H_j(\bm \phi)}{H_j(\bm \phi)}
\leq&
\sum_{a<b}\pi^{(a)}\pi^{(b)}
\left[
\gamma\ell
(v_a^{(j)}-v_b^{(j)})
(\mu^{(a)}-\mu^{(b)})
+
\frac{\gamma^2\ell^2s_{ab}}{2}
(v_a^{(j)}-v_b^{(j)})^2
\right] \\
= & \pi^{(1)}\pi^{(j)}
\left[
-2\gamma\ell\bigl(\mu^{(1)}-\mu^{(j)}\bigr)
+2\gamma^2\ell^2s_{1j}
\right]+\sum_{\substack{k=2\\k\neq j}}^d
\pi^{(1)}\pi^{(k)}
\left[
-\gamma\ell\bigl(\mu^{(1)}-\mu^{(k)}\bigr)
+\frac{\gamma^2\ell^2s_{1k}}{2}
\right]\\
& + \sum_{k=2}^{j-1}
\pi^{(j)}\pi^{(k)}
\left[
-\gamma\ell\bigl(\mu^{(k)}-\mu^{(j)}\bigr)
+\frac{\gamma^2\ell^2s_{kj}}{2}
\right] +\sum_{k=j+1}^d
\pi^{(j)}\pi^{(k)}
\left[
\gamma\ell\bigl(\mu^{(j)}-\mu^{(k)}\bigr)
+\frac{\gamma^2\ell^2s_{jk}}{2}
\right].
\end{aligned}
\label{eq:mab-fractional-ratio-generator-upper-bound}
\end{equation}
For any $x\geq\rho$, the property of $\gamma$ in
\eqref{eq:mab-fractional-exponent} and $s_{ab}\leq\overline s$,
$\gamma\ell\overline s\leq\rho/2$ imply 
\begin{equation*}
\begin{aligned}
-2\gamma\ell x+2\gamma^2\ell^2s_{ab}
\le-\gamma\ell\rho = -2a_{\gamma},\quad
-\gamma\ell x+\frac{\gamma^2\ell^2s_{ab}}{2}
\le-\frac{3}{4}\gamma\ell\rho \leq -a_{\gamma}.
\end{aligned}
\label{eq:mab-fractional-negative-coefficients}
\end{equation*}
For $j<k$, the definition of $\Delta_{\max}$ gives
\begin{equation*}
\gamma\ell\bigl(\mu^{(j)}-\mu^{(k)}\bigr)
+\frac{\gamma^2\ell^2s_{jk}}{2}
\le b_\gamma.
\label{eq:mab-fractional-positive-coefficient}
\end{equation*}
All terms involving the best arm in \eqref{eq:mab-fractional-ratio-generator-upper-bound} are bounded above by $-a_{\gamma}\pi^{(1)}(1 - \pi^{(1)})$.  The third term on the right-hand side of \eqref{eq:mab-fractional-ratio-generator-upper-bound} is nonpositive. Hence, we conclude that
\begin{equation}
\mathcal L_{\bm\sigma}H_j(\bm\phi)
\le
H_j(\bm\phi)\left(
-a_\gamma\pi^{(1)}(1-\pi^{(1)})
+b_\gamma\pi^{(j)}\Lambda_j(\bm\pi)
\right),
\qquad 2\le j\le d.
\label{eq:mab-fractional-triangular-generator-bound}
\end{equation}
It\^{o}'s formula also gives the semimartingale decomposition
\begin{equation}
\begin{aligned}
H_j(\bm\phi_t)
&=H_j(\bm\phi_0)
+\int_0^t\mathcal L_{\bm\sigma}H_j(\bm\phi_s)\,\dd s
+M_t^{j,\gamma},
\end{aligned}
\label{eq:mab-fractional-ratio-semimartingale}
\end{equation}
where $M^{j,\gamma}_t :=\gamma\ell\int_0^t
H_j(\bm\phi_s)(\bm e_j-\bm e_1)^\top
\bm G_{\bm\sigma}(\bm\pi_s)\,\dd\bm B_s^{\phi}$ is a continuous local martingale. 

We first verify the integrability needed to apply
Lemma~\ref{lem:mab-integrating-factor-supermartingale} at a deterministic
time.  Since $0\le\pi^{(j)}\Lambda_j(\bm\pi)\le1$,
\eqref{eq:mab-fractional-triangular-generator-bound} implies
$\mathcal L_{\bm\sigma}H_j\le b_\gamma H_j$.  For $m\in\mathbb N$, let
$\zeta_m:=\inf\{t\geq0:|\bm\phi_t|_2\geq m\}$.  By
Lemma~\ref{lem:mab-wellposedness}, $\zeta_m\uparrow\infty$ almost surely.  The
stochastic integral in \eqref{eq:mab-fractional-ratio-semimartingale}, stopped
at $\zeta_m$, is a martingale.  Hence It\^{o}'s formula and the drift bound
give, for every $t<\infty$,
\begin{equation}
\begin{aligned}
\E[H_j(\bm\phi_{t\wedge\zeta_m})]
&\le H_j(\bm\phi_0)
+b_\gamma\E\!\left[\int_0^{t\wedge \zeta_m}
H_j(\bm\phi_s) \dd s
\right]  \\
&\le H_j(\bm\phi_0)
+b_\gamma\int_0^t
\E[H_j(\bm\phi_{s\wedge\zeta_m})]\,\dd s.
\end{aligned}
\label{eq:mab-fractional-ratio-stopped-moment}
\end{equation}
Applying Gronwall's inequality to
\eqref{eq:mab-fractional-ratio-stopped-moment}, followed by Fatou's lemma as
$m\to\infty$, proves
\begin{equation}
\E[H_j(\bm\phi_t)]
\le H_j(\bm\phi_0)e^{b_\gamma t}<\infty,
\qquad t<\infty.
\label{eq:mab-fractional-ratio-finite-moment}
\end{equation}
Thus, $ H_j(\phi_t) \in L^1$ for every $ t>0$.

We now state explicitly how that lemma will be used.  Since
$\pi^{(j)}\le1$, 
\eqref{eq:mab-fractional-triangular-generator-bound} implies
\begin{equation}
\mathcal L_{\bm\sigma}H_j
\le H_j\left[
b_\gamma\Lambda_j(\bm\pi)
-a_\gamma\pi^{(1)}(1-\pi^{(1)})
\right].
\label{eq:mab-fractional-integrating-factor-form}
\end{equation}
Apply Lemma~\ref{lem:mab-integrating-factor-supermartingale} to
\eqref{eq:mab-fractional-ratio-semimartingale} and
\eqref{eq:mab-fractional-integrating-factor-form} with
$ X_t=H_j(\bm\phi_t),
\ 
\alpha_t=b_\gamma\Lambda_j(\bm\pi_t),
\ 
c=a_\gamma,
\ 
g_t=\pi_t^{(1)}(1-\pi_t^{(1)})$. 
It follows that, almost surely on the event
$\{\int_0^\infty\Lambda_j(\bm\pi_t)\,\dd t<\infty\}$,
\begin{equation}
H_j(\bm\phi_t)\ \text{converges to a finite limit, and }
\qquad
\int_0^\infty
H_j(\bm\phi_t)\pi_t^{(1)}(1-\pi_t^{(1)})\,\dd t<\infty.
\label{eq:mab-fractional-ratio-comparison-consequence}
\end{equation}

We claim that
\begin{equation}
\int_0^\infty\pi_t^{(1)}\,\dd t=\infty
\qquad\text{almost surely}.
\label{eq:mab-optimal-arm-infinite-occupation}
\end{equation}
We prove this claim by contradiction. Suppose otherwise, that the event $ E:=\left\{
\int_0^\infty\pi_t^{(1)}\,\dd t<\infty
\right\}$
has positive probability.  Since $\Lambda_d(\bm\pi)=0$, \eqref{eq:mab-fractional-ratio-comparison-consequence} shows that
$H_d(\bm\phi_t)$ converges to a finite limit almost surely. Recall $H_d(\bm\phi_t)$ has continuous paths, so it is also bounded on $[0,\infty)$.  Hence, almost surely on
$E$, the identity in
\eqref{eq:mab-fractional-ratio-and-lower-mass} consequently gives
\begin{equation}
\int_0^\infty\pi_t^{(d)}\,\dd t
=\int_0^\infty
\pi_t^{(1)}H_d(\bm\phi_t)^{1/\gamma}\,\dd t
\le
\sup_{t\geq0}H_d(\bm\phi_t)^{1/\gamma}
\int_0^\infty\pi_t^{(1)}\,\dd t
<\infty.
\label{eq:mab-descending-integrability-base}
\end{equation}

We next use induction. Now let $j\in\{2,\ldots,d-1\}$ and suppose, in descending order, that
$\int_0^\infty\pi_t^{(k)}\,\dd t<\infty$ almost surely on $E$ for every
$k>j$.  Then
\begin{equation}
\int_0^\infty\Lambda_j(\bm\pi_t)\,\dd t
=\sum_{k=j+1}^d\int_0^\infty\pi_t^{(k)}\,\dd t
<\infty
\qquad\text{almost surely on }E.
\label{eq:mab-descending-lower-mass-integrability}
\end{equation}
Applying \eqref{eq:mab-fractional-ratio-comparison-consequence} on the event in
\eqref{eq:mab-descending-lower-mass-integrability} shows that
$H_j(\bm\phi_t)$ converges to a finite limit and is therefore bounded, almost
surely on $E$.  Repeating the
calculation in \eqref{eq:mab-descending-integrability-base} proves
\begin{equation}
\int_0^\infty\pi_t^{(j)}\,\dd t<\infty
\qquad\text{almost surely on }E.
\label{eq:mab-descending-integrability-step}
\end{equation}
Descending induction using
\eqref{eq:mab-descending-integrability-base} and
\eqref{eq:mab-descending-integrability-step} proves that every arm has finite
occupation time almost surely on $E$. This is a contradiction, since there are only finitely many arms, and we have
\begin{equation}
\infty > \sum_{i=1}^d\int_0^\infty\pi_t^{(i)}\,\dd t
=\int_0^\infty\sum_{i=1}^d\pi_t^{(i)}\,\dd t
=\int_0^\infty1\,\dd t
=\infty.
\label{eq:mab-total-occupation-contradiction}
\end{equation}
Thus $\mathbb P(E)=0$, which proves
\eqref{eq:mab-optimal-arm-infinite-occupation}.

We next prove
\begin{equation}
H_j(\bm\phi_t)\longrightarrow0,
\quad 2\le j\le d,\quad\text{almost surely},
\label{eq:mab-fractional-ratios-vanish}
\end{equation}
by descending induction.  Since $\Lambda_d(\bm\pi)=0$,
\eqref{eq:mab-fractional-ratio-comparison-consequence} gives
\begin{equation}
H_d(\bm\phi_t)\longrightarrow H_{d,\infty}<\infty,
\text{ and }
\int_0^\infty
H_d(\bm\phi_t)\pi_t^{(1)}(1-\pi_t^{(1)})\,\dd t<\infty
\label{eq:mab-last-arm-limit-and-occupation}
\end{equation}
almost surely.  On the event $\{H_{d,\infty}>0\}$, there exists $T_0<\infty$ and $c>0$ such that, for all $t\geq T_0$,
\begin{equation}
H_d(\bm\phi_t)\geq\frac{H_{d,\infty}}{2},
\text{ and }
\frac{\pi_t^{(d)}}{\pi_t^{(1)}}
=H_d(\bm\phi_t)^{1/\gamma}\geq c.
\label{eq:mab-positive-fractional-limit-ratio}
\end{equation}
The second inequality in
\eqref{eq:mab-positive-fractional-limit-ratio} implies
\begin{equation}
1-\pi_t^{(1)}
\geq\pi_t^{(d)}
\geq c\pi_t^{(1)},
\text{ hence, }
1-\pi_t^{(1)}\geq\frac{c}{1+c},
\qquad t\geq T_0.
\label{eq:mab-positive-limit-away-from-one}
\end{equation}
Combining
\eqref{eq:mab-positive-fractional-limit-ratio} and
\eqref{eq:mab-positive-limit-away-from-one}, the integrand in
\eqref{eq:mab-last-arm-limit-and-occupation} is bounded below, for
$t\geq T_0$, by
\begin{equation}
H_d(\bm\phi_t)\pi_t^{(1)}(1-\pi_t^{(1)})
\geq
\frac{H_{d,\infty}}{2}\frac{c}{1+c}\pi_t^{(1)}.
\label{eq:mab-positive-limit-integrand-lower-bound}
\end{equation}
This contradicts
\eqref{eq:mab-optimal-arm-infinite-occupation}.  Hence
$H_{d,\infty}=0$ almost surely, proving the base case of
\eqref{eq:mab-fractional-ratios-vanish}.

Fix $j\in\{2,\ldots,d-1\}$ and suppose that
$H_k(\bm\phi_t)\to0$ as $ t \rightarrow \infty$ almost surely for all $k>j$.  By
\eqref{eq:mab-fractional-ratio-and-lower-mass},
\begin{equation}
Q_j(t):=\frac{\Lambda_j(\bm\pi_t)}{\pi_t^{(1)}}
=\sum_{k=j+1}^d
\frac{\pi_t^{(k)}}{\pi_t^{(1)}}
=\sum_{k=j+1}^dH_k(\bm\phi_t)^{1/\gamma}
\longrightarrow0\quad \text{as}\ \ t\rightarrow \infty
\quad\text{almost surely}.
\label{eq:mab-lower-arm-relative-mass-vanishes}
\end{equation}
Let
\begin{equation}
\varepsilon:=\frac{a_\gamma}{2b_\gamma}>0,
\qquad
\tau_{j,n}:=\inf\left\{
t\geq n:Q_j(t)>\varepsilon
\right\},
\qquad n\in\mathbb N,
\label{eq:mab-random-tail-stopping-times}
\end{equation}
with $\inf\varnothing=\infty$.  The process $Q_j$ is continuous and adapted, so
$\tau_{j,n}$ is a stopping time.  On $\{\tau_{j,n}>n\}$, we have
$Q_j(t)\le\varepsilon$, and hence
$\Lambda_j(\bm\pi_t)\le\varepsilon\pi_t^{(1)}$, for
$n\le t<\tau_{j,n}$.  If $\tau_{j,n}=n$, the process stopped at
$\tau_{j,n}$ is constant from its starting time and the conclusion below is
immediate.  Since
$\pi^{(j)}\le1-\pi^{(1)}$, it follows from
\eqref{eq:mab-fractional-drift-constants} and
\eqref{eq:mab-random-tail-stopping-times} that
\begin{equation*}
b_\gamma\pi_t^{(j)}\Lambda_j(\bm\pi_t)
\le b_\gamma(1-\pi_t^{(1)})\Lambda_j(\bm\pi_t)
\le\frac{a_\gamma}{2}\pi_t^{(1)}(1-\pi_t^{(1)}),
\qquad n\le t<\tau_{j,n}.
\label{eq:mab-random-tail-absorption}
\end{equation*}
Consequently, the generator estimate
\eqref{eq:mab-fractional-triangular-generator-bound} strengthens on this
stopped interval to
\begin{equation*}
\mathcal L_{\bm\sigma}H_j
\le-\frac{a_\gamma}{2}
H_j\pi^{(1)}(1-\pi^{(1)}).
\label{eq:mab-random-tail-negative-generator}
\end{equation*}
By \eqref{eq:mab-fractional-ratio-finite-moment}, $H_j(\bm\phi_n)\in L^1$.
Applying Lemma~\ref{lem:mab-integrating-factor-supermartingale} with
$s=n$, $\tau=\tau_{j,n}$, $X_t=H_j(\bm\phi_t)$,
$\alpha_t=0$, $c=a_\gamma/2$, and
$g_t=\pi_t^{(1)}(1-\pi_t^{(1)})$, we obtain
\begin{equation}
\lim_{t\to\infty}H_j(\bm\phi_{t\wedge\tau_{j,n}})
=:H_{j,n,\infty}<\infty,
\text{ and }
\int_n^{\tau_{j,n}}
H_j(\bm\phi_t)\pi_t^{(1)}(1-\pi_t^{(1)})\,\dd t<\infty
\label{eq:mab-stopped-tail-convergence-and-occupation}
\end{equation}
almost surely.  The convergence in
\eqref{eq:mab-lower-arm-relative-mass-vanishes} implies
\begin{equation}
\mathbb P\left(
\bigcup_{n=1}^\infty\{\tau_{j,n}=\infty\}
\right)=1.
\label{eq:mab-eventual-tail-union}
\end{equation}
Indeed, on every path on which $Q_j(t)\to0$, some integer $n$ satisfies
$Q_j(t)\le\varepsilon$ for all $t\geq n$, and hence
$\tau_{j,n}=\infty$.  

For each $n\in\mathbb N$, let $\Omega_{j,n}$ be the probability-one event on which \eqref{eq:mab-stopped-tail-convergence-and-occupation} holds, and define $\Omega_j:=\bigl(\bigcap_{n=1}^{\infty}\Omega_{j,n}\bigr)\cap\bigl(\bigcup_{n=1}^{\infty}\{\tau_{j,n}=\infty\}\bigr)$. Since $\mathbb N$ is countable, \eqref{eq:mab-stopped-tail-convergence-and-occupation} and \eqref{eq:mab-eventual-tail-union} imply that $\mathbb P(\Omega_j)=1$. Fix $\omega\in\Omega_j$. Then there exists an integer $n=n(\omega)$ such that $\tau_{j,n}(\omega)=\infty$. For this $n$, $t\wedge\tau_{j,n}(\omega)=t$, and hence $\lim_{t\to\infty}H_j(\bm\phi_t(\omega))=\lim_{t\to\infty}H_j(\bm\phi_{t\wedge\tau_{j,n}}(\omega))=H_{j,n,\infty}(\omega)<\infty$. Moreover, $\int_n^\infty H_j(\bm\phi_t(\omega))\pi_t^{(1)}(\omega)(1-\pi_t^{(1)}(\omega))\,\dd t=\int_n^{\tau_{j,n}(\omega)}H_j(\bm\phi_t(\omega))\pi_t^{(1)}(\omega)(1-\pi_t^{(1)}(\omega))\,\dd t<\infty$. Since the integrand is continuous in $t$, its integral over the compact interval $[0,n]$ is also finite. Therefore, $\int_0^\infty H_j(\bm\phi_t)\pi_t^{(1)}(1-\pi_t^{(1)})\,\dd t<\infty$ almost surely.
If the limit of $H_j(\bm\phi_t)$ were positive, repeating the argument in
\eqref{eq:mab-positive-fractional-limit-ratio}--
\eqref{eq:mab-positive-limit-integrand-lower-bound}, with $j$ in place of
$d$, would contradict
\eqref{eq:mab-optimal-arm-infinite-occupation}.  Therefore the limit is zero.
Descending induction proves
\eqref{eq:mab-fractional-ratios-vanish} for every $j\geq2$.

Equations \eqref{eq:mab-fractional-ratio-and-lower-mass} and
\eqref{eq:mab-fractional-ratios-vanish} give $ \frac{\pi_t^{(j)}}{\pi_t^{(1)}}
=H_j(\bm\phi_t)^{1/\gamma}\rightarrow0,
\ 
\pi_t^{(1)}
=\left(1+\sum_{j=2}^d
\frac{\pi_t^{(j)}}{\pi_t^{(1)}}\right)^{-1}
\rightarrow1$ almost surely.  This proves almost sure convergence of the policy.  \Halmos
\end{proof}

Theorem~\ref{thm:mab-heterogeneous-arbitrary-learning-rate-as-convergence}
removes the learning-rate restriction only from the pathwise convergence
statement.  It does not imply logarithmic expected regret for arbitrary
$\ell$: rare paths can prevent the uniform integrability needed for such a
conclusion.  The proof also uses $\rho>0$ and therefore does not cover tied
suboptimal means.

\subsection{Discrete-time policy gradient algorithm}
\label{subsubsec:discrete-mab-direct-analysis}

We now return to
the conventional discrete-time stochastic-gradient algorithm, and present a new proof for the convergence and regret based on the same Lyapunov function discovered in our earlier analysis for the policy gradient SDE.  

Let $(\mathcal F_n)_{n\geq0}$ be the information available immediately before
round $n$.  The vectors $\bm\phi_n$ and $\bm\pi_n$ and the scalar baseline
$B_n$ are $\mathcal F_n$-measurable, with $ \pi_n^{(a)}
=\frac{e^{\phi_n^{(a)}}}{\sum_{j=1}^d e^{\phi_n^{(j)}}},
\ A_n\mid\mathcal F_n\sim\bm\pi_n.$
After sampling $A_n$, an arm-dependent reward $Y_n$ is observed.  Its
conditional law may vary with the selected arm and with the past, and we only
require $ \E[Y_n\mid\mathcal F_n,A_n=a]=\mu^{(a)}.$
The baseline $B_n$ is predictable and action-independent: It may depend on all past
observations, but not on the action or reward at round $n$.  Put
$X_n:=Y_n-B_n$.  The direct discrete-time update is
\begin{equation}
\bm\phi_{n+1}
=\bm\phi_n+\ell X_n(\bm e_{A_n}-\bm\pi_n).
\label{eq:discrete-mab-actor-update}
\end{equation}
Retain the unique-best-arm convention and the quantities
$\Delta_i,\Delta_{\min},\Delta_{\max}$ and $I=\{2,\ldots,d\}$.  Define
\begin{equation}
K=\Delta_{\max}+\frac{\Delta_{\min}}{2d},
\qquad
C=\frac{2K}{\Delta_{\min}},
\qquad
V(\bm\phi)
=\sum_{\varnothing\neq A\subseteq I}
C^{|A|-1}h_A(\bm\phi)
=
\sum_{\varnothing\neq A\subseteq I}
C^{|A|-1}\prod_{i\in A}e^{\phi^{(i)}-\phi^{(1)}}.
\label{eq:discrete-mab-subset-potential}
\end{equation}
Define the conditional instantaneous regret and expected regret by
\begin{equation}
r_n
:=\mu^{(1)}-\bm\mu^\top\bm\pi_n
=\sum_{i=2}^d\Delta_i\pi_n^{(i)},
\qquad
\mathcal R_N
:=\E\left[
\sum_{n=0}^{N-1}
\bigl(\mu^{(1)}-\mu^{(A_n)}\bigr)
\right]
=\E\left[\sum_{n=0}^{N-1}r_n\right].
\label{eq:discrete-mab-regret-definition}
\end{equation}

\begin{theorem}
\label{thm:discrete-mab-horizon-independent-learning-rate}
Suppose that arm $1$ is uniquely optimal.   For $q\in\{1,d\}$, suppose that
there are deterministic, arm-dependent constants $\Gamma_{q,a}<\infty$ such
that, almost surely for every $n$ and $a$,
\begin{equation}
\E\left[
X_n^2e^{\ell q|X_n|}
\mid\mathcal F_n,A_n=a
\right]
\leq\Gamma_{q,a}.
\label{eq:discrete-mab-arm-dependent-exponential-moments}
\end{equation}
Put $ \Gamma_q=\max_{1\leq a\leq d}\Gamma_{q,a},
\  q\in\{1,d\}.$
Take $\bm\phi_0=\bm0$ and assume that the fixed learning rate satisfies $ \ell d\Gamma_d\le\Delta_{\min}.$
Then 
\begin{equation}
\pi_n^{(1)}\longrightarrow1,
\qquad
\pi_n^{(i)}\longrightarrow0,\quad 2\le i\le d
\quad\text{almost surely}.
\label{eq:discrete-mab-almost-sure-convergence}
\end{equation}
Moreover, for every integer $N\geq0$,
\begin{equation}
\begin{aligned}
\mathcal R_N
\le
\frac{d-1}{\ell}
\log\left(
1+
\frac{\ell\Delta_{\max}+\frac12\ell^2\Gamma_1}{d}N
\right)+
\frac{\Delta_{\max}(d-1)}
{\ell\Delta_{\min}}
\frac{(1+C)^{d-1}-1}{C}.
\end{aligned}
\label{eq:discrete-mab-anytime-regret-bound}
\end{equation}
\end{theorem}

\begin{proof}{Proof}
We prove a slightly stronger estimate for an arbitrary finite deterministic
$\bm\phi_0$ and specialize to $\bm\phi_0=\bm0$ at the end.  Write
$\E_n[\cdot]=\E[\cdot\mid\mathcal F_n]$.
Fix $\bm v\in\mathbb R^d$ and define $ f_{\bm v}(\bm\phi):=e^{\bm v^\top\bm\phi},
\
\bar v_n:=\bm\pi_n^\top\bm v,
\
\operatorname{osc}(\bm v):=\max_a v_a-\min_a v_a.$
By \eqref{eq:discrete-mab-actor-update} and Taylor's formula, we obtain
\begin{equation}
\begin{aligned}
f_{\bm v}(\bm\phi_{n+1})-f_{\bm v}(\bm\phi_n)
\le{}&
f_{\bm v}(\bm\phi_n)
\ell X_n(v_{A_n}-\bar v_n)\\
&+
\frac{\ell^2}{2}f_{\bm v}(\bm\phi_n)
X_n^2(v_{A_n}-\bar v_n)^2
e^{\ell|X_n||v_{A_n}-\bar v_n|}.
\end{aligned}
\label{eq:discrete-mab-pathwise-exponential-bound}
\end{equation}

Denote
$\bm J_n=\diag\{\bm\pi_n\}-\bm\pi_n\bm\pi_n^\top$.
Because $B_n$ is $\mathcal F_n$-measurable and action-independent,
\begin{equation}
\begin{aligned}
\E_n[X_n(v_{A_n}-\bar v_n)]
&=
\sum_{a=1}^d
\pi_n^{(a)}(\mu^{(a)}-B_n)(v_a-\bar v_n)\\
&=
\sum_{a=1}^d
\pi_n^{(a)}\mu^{(a)}(v_a-\bar v_n)
-B_n\sum_{a=1}^d
\pi_n^{(a)}(v_a-\bar v_n)=\bm v^\top\bm J_n\bm\mu.
\end{aligned}
\label{eq:discrete-mab-baseline-cancellation}
\end{equation}
If $\operatorname{osc}(\bm v)\leq q$ for $q\in\{1,d\}$, then
\eqref{eq:discrete-mab-arm-dependent-exponential-moments} and
$ |v_a-\bar v_n|\le\operatorname{osc}(\bm v)$ imply
\begin{equation}
\begin{aligned}
&\E_n\left[
X_n^2(v_{A_n}-\bar v_n)^2
e^{\ell|X_n||v_{A_n}-\bar v_n|}
\right] =  \sum_{a=1}^d \E_n\left[
X_n^2(v_{A_n}-\bar v_n)^2
e^{\ell|X_n||v_{A_n}-\bar v_n|}
\mid A_n = a \right] \pi_n^{(a)}  \\
\leq &  
\Gamma_q
\sum_{a=1}^d\pi_n^{(a)}(v_a-\bar v_n)^2
=
\Gamma_q\bm v^\top\bm J_n\bm v.
\end{aligned}
\label{eq:discrete-mab-quadratic-form}
\end{equation}
Taking conditional expectations in
\eqref{eq:discrete-mab-pathwise-exponential-bound} and using
\eqref{eq:discrete-mab-baseline-cancellation}--
\eqref{eq:discrete-mab-quadratic-form} yields
\begin{equation}
\E_n[f_{\bm v}(\bm\phi_{n+1})]-f_{\bm v}(\bm\phi_n)
\le
f_{\bm v}(\bm\phi_n)
\left[
\ell\bm v^\top\bm J_n\bm\mu
+
\frac{\ell^2}{2}
\Gamma_q
\bm v^\top\bm J_n\bm v
\right].
\label{eq:discrete-mab-exponential-one-step}
\end{equation}
Recall from \eqref{eq:mab-pairwise-J-identity} the identity
\begin{equation}
\bm v^\top\bm J_n\bm w
=
\sum_{1\le a<b\le d}
\pi_n^{(a)}\pi_n^{(b)}
(v_a-v_b)(w_a-w_b).
\label{eq:discrete-mab-pairwise-J}
\end{equation}
Following the same idea as in the proof of Theorem \ref{thm:mab-heterogeneous-volatility-regret}, for a fixed nonempty $A\subseteq I$, write $m=|A|$, and define $ v_1^A=-m,
\
v_i^A=\one_{\{i\in A\}},
\ 2\le i\le d.$
Then $h_A=e^{(\bm v^A)^\top\bm\phi}$ and
$\operatorname{osc}(\bm v^A)=m+1\le d$.  Equations
\eqref{eq:discrete-mab-exponential-one-step} and
\eqref{eq:discrete-mab-pairwise-J} imply
\begin{equation}
\begin{aligned}
\frac{\E_n[h_A(\bm\phi_{n+1})]-h_A(\bm\phi_n)}
{h_A(\bm\phi_n)}
\le
\sum_{a<b}\pi_n^{(a)}\pi_n^{(b)}
\Bigg[
&
\ell(v_a^A-v_b^A)(\mu^{(a)}-\mu^{(b)})+
\frac{\ell^2\Gamma_d}{2}
(v_a^A-v_b^A)^2
\Bigg].
\end{aligned}
\label{eq:discrete-mab-one-subset-pairwise}
\end{equation}

Using the same method as in the proof of Theorem \ref{thm:mab-heterogeneous-volatility-regret}, we obtain
\begin{equation}
\begin{aligned}
&\E_n[h_A(\bm\phi_{n+1})]-h_A(\bm\phi_n)\\
&\quad\leq
-\frac{\ell\Delta_{\min}}{2}
h_A\pi_n^{(1)}
\left[
(m+1)\sum_{i\in A}\pi_n^{(i)}
+
m\sum_{j\in I\setminus A}\pi_n^{(j)}
\right]+
\ell K h_A
\sum_{\substack{i\in A\\j\in I\setminus A}}
\pi_n^{(i)}\pi_n^{(j)}.
\end{aligned}\notag
\end{equation}
Thus, by the definition of the Lyapunov function $ V(\phi)$, we have 
\begin{equation}
\begin{aligned}
&\E_n[V(\bm\phi_{n+1})]-V(\bm\phi_n)\\
&\quad\leq
-\ell\Delta_{\min}\sum_{i=2}^d(\pi_n^{(i)})^2
+\ell\sum_{r=2}^{d-1}C^{r-2}
\left[
K(r-1)-\frac{\Delta_{\min}C}{2}(r+1)
\right]
\sum_{\substack{D\subseteq I\\|D|=r}}
h_D(\bm\phi_n)\pi_n^{(1)}\sum_{i\in D}\pi_n^{(i)}\\
&\quad=
-\ell\Delta_{\min}\sum_{i=2}^d(\pi_n^{(i)})^2
-2\ell K\sum_{r=2}^{d-1}C^{r-2}
\sum_{\substack{D\subseteq I\\|D|=r}}
h_D(\bm\phi_n)\pi_n^{(1)}\sum_{i\in D}\pi_n^{(i)}\\
&\quad\leq
-\ell\Delta_{\min}\sum_{i=2}^d(\pi_n^{(i)})^2.
\end{aligned}
\label{eq:discrete-mab-composite-drift-cancellation}
\end{equation}
Here the first term is the singleton contribution, because
$h_{\{i\}}\pi_n^{(1)}\pi_n^{(i)}=(\pi_n^{(i)})^2$, and the equality uses
$\Delta_{\min}C/2=K$.  

Summing
\eqref{eq:discrete-mab-composite-drift-cancellation} from $n=0$ to $N-1$,
taking expectations, and using $V\geq0$ gives $ \E\left[
\sum_{n=0}^{N-1}\sum_{i=2}^d(\pi_n^{(i)})^2
\right]
\le
\frac{V(\bm\phi_0)}{\ell\Delta_{\min}}.$
Letting $N\to \infty$ yields
\begin{equation}
\E\left[
\sum_{n=0}^{\infty}\sum_{i=2}^d(\pi_n^{(i)})^2
\right]
\le
\frac{V(\bm\phi_0)}{\ell\Delta_{\min}}.
\label{eq:discrete-mab-general-infinite-occupation}
\end{equation}
The nonnegative series in
\eqref{eq:discrete-mab-general-infinite-occupation} has finite expectation and
is therefore finite almost surely.  Hence
$\pi_n^{(i)}\to0$ for every $i\geq2$, and normalization gives
$\pi_n^{(1)}\to1$. 

The update conserves the sum of the logits pathwise:
\begin{equation}
\sum_{a=1}^d\phi_{n+1}^{(a)}
=
\sum_{a=1}^d\phi_n^{(a)}
+
\ell X_n
\left(
1-\sum_{a=1}^d\pi_n^{(a)}
\right)
=S_0,
\qquad
S_0:=\sum_{a=1}^d\phi_0^{(a)}.
\notag
\end{equation}
Fix $i\geq2$ and apply
\eqref{eq:discrete-mab-exponential-one-step} with
$\bm v=-\bm e_i$, whose oscillation is one.  Since $ -\bm e_i^\top\bm J_n\bm\mu
=
\pi_n^{(i)}
(\bm\mu^\top\bm\pi_n - \mu^{(i)})
\le\Delta_{\max}\pi_n^{(i)}
$
and $ \bm e_i^\top\bm J_n\bm e_i
=\pi_n^{(i)}(1-\pi_n^{(i)})
\le\pi_n^{(i)},$
we obtain
\begin{equation}
\E_n[e^{-\phi_{n+1}^{(i)}}]-e^{-\phi_n^{(i)}}
\le
e^{-\phi_n^{(i)}}\pi_n^{(i)}
\left(
\ell\Delta_{\max}+\frac{\ell^2\Gamma_1}{2}
\right).
\notag
\end{equation}
Let $Z_n:=\sum_{j=1}^d e^{\phi_n^{(j)}}$.  Then
$e^{-\phi_n^{(i)}}\pi_n^{(i)}=1/Z_n \le \frac{1}{d}\exp\left( -\frac{1}{d}\sum_{j=1}^{d}\phi_n^{(j)}\right) = \frac{e^{-S_0/d}}{d}$.  
Thus,
\begin{equation}
\E[e^{-\phi_N^{(i)}}]
\le
e^{-\phi_0^{(i)}}+a_{\ell,\mathrm{disc}}N,
\label{eq:discrete-mab-negative-logit-moment}
\end{equation}
where $ a_{\ell,\mathrm{disc}}
:=
\frac{e^{-S_0/d}}{d}
\left(
\ell\Delta_{\max}+\frac{\ell^2\Gamma_1}{2}
\right)$.
Conservation also yields
\begin{equation}
\phi_N^{(1)}-\phi_0^{(1)}
=
\sum_{i=2}^d
(\phi_0^{(i)}-\phi_N^{(i)}).
\label{eq:discrete-mab-best-logit-conservation}
\end{equation}
Since the exponential factor in
\eqref{eq:discrete-mab-arm-dependent-exponential-moments} is at least one,
conditional Cauchy--Schwarz gives
$\E_n[|X_n|]\leq\sqrt{\Gamma_d}$.  The update
\eqref{eq:discrete-mab-actor-update} therefore implies, by induction over the
finite number of steps, that every coordinate of $\bm\phi_N$ is integrable.
For each $i\geq2$, Jensen's inequality and
\eqref{eq:discrete-mab-negative-logit-moment} imply
\begin{equation}
\begin{aligned}
\E[\phi_0^{(i)}-\phi_N^{(i)}]
&=
\E\left[
\log\left(
e^{\phi_0^{(i)}}e^{-\phi_N^{(i)}}
\right)
\right]\le
\log\left(
e^{\phi_0^{(i)}}
\E[e^{-\phi_N^{(i)}}]
\right)\le
\log\left(
1+a_{\ell,\mathrm{disc}}e^{\phi_0^{(i)}}N
\right).
\end{aligned}
\label{eq:discrete-mab-jensen-logit-bound}
\end{equation}
Summing \eqref{eq:discrete-mab-jensen-logit-bound} and using
\eqref{eq:discrete-mab-best-logit-conservation} gives
\begin{equation}
\E[\phi_N^{(1)}-\phi_0^{(1)}]
\le
\sum_{i=2}^d
\log\left(
1+a_{\ell,\mathrm{disc}}e^{\phi_0^{(i)}}N
\right).
\label{eq:discrete-mab-best-logit-bound}
\end{equation}

The conditional drift of the best coordinate is $ \E_n[\phi_{n+1}^{(1)}-\phi_n^{(1)}]
=
\ell
\sum_{a=1}^d
\pi_n^{(a)}(\mu^{(a)}-B_n)
(\one_{\{a=1\}}-\pi_n^{(1)})=
\ell\pi_n^{(1)}
(\mu^{(1)}-\bm\mu^\top\bm\pi_n)
=\ell\pi_n^{(1)}r_n.$
Therefore,
\begin{equation}
\E[\phi_N^{(1)}-\phi_0^{(1)}]
=
\ell
\E\left[
\sum_{n=0}^{N-1}\pi_n^{(1)}r_n
\right].
\label{eq:discrete-mab-best-coordinate-telescope}
\end{equation}
Let
$u_n:=1-\pi_n^{(1)}=\sum_{i=2}^d\pi_n^{(i)}$.
Then $r_n\le\Delta_{\max}u_n$ and, by Cauchy--Schwarz, $ u_n^2\le(d-1)\sum_{i=2}^d(\pi_n^{(i)})^2.$
Consequently, $ r_n
=\pi_n^{(1)}r_n+u_nr_n\le
\pi_n^{(1)}r_n
+
\Delta_{\max}(d-1)
\sum_{i=2}^d(\pi_n^{(i)})^2.$
Summing $ r_n$, taking expectations,
and applying \eqref{eq:discrete-mab-general-infinite-occupation},
\eqref{eq:discrete-mab-best-logit-bound}, and
\eqref{eq:discrete-mab-best-coordinate-telescope}, we conclude that
\begin{equation}
\mathcal R_N
\le
\frac1\ell
\sum_{i=2}^d
\log\left(
1+a_{\ell,\mathrm{disc}}e^{\phi_0^{(i)}}N
\right)
+
\frac{\Delta_{\max}(d-1)}
{\ell\Delta_{\min}}
V(\bm\phi_0).
\label{eq:discrete-mab-general-initial-regret}
\end{equation}

If $\bm\phi_0=\bm0$, then $S_0=0$, $ a_{\ell,\mathrm{disc}}
=\frac1d
\left(
\ell\Delta_{\max}+\frac{\ell^2\Gamma_1}{2}
\right),
\ V(\bm0)
=
\sum_{m=1}^{d-1}\binom{d-1}{m}C^{m-1}
=\frac{(1+C)^{d-1}-1}{C}.$
This proves
\eqref{eq:discrete-mab-anytime-regret-bound}.   \Halmos
\end{proof}

\begin{remark}
\label{rmk:discrete-mab-exponential-moment-extension}
The zero initialization in
Theorem~\ref{thm:discrete-mab-horizon-independent-learning-rate} only
simplifies the constants.  For any finite deterministic $\bm\phi_0$, the same
conclusions hold with \eqref{eq:discrete-mab-general-infinite-occupation} and
\eqref{eq:discrete-mab-general-initial-regret}.  When $Y_n,B_n\in[0,1]$, the
exact condition is $\ell d e^{\ell d}\leq\Delta_{\min}$, while
$\ell\leq\Delta_{\min}/(2d)$ is a simpler sufficient condition. 
\end{remark}

\begin{remark}
\label{rmk:discrete-mab-gaussian-rewards}
Suppose that $B_n=0$ and, conditionally on $\mathcal F_n$ and $A_n=a$,
$Y_n\sim\mathcal N(\mu^{(a)},s_a)$, where
$s_a=(\sigma^{(a)})^2$.  Define $ \nu_{G,\bm\sigma}
:=
2\max_{1\leq a\leq d}
\left\{
\left((|\mu^{(a)}|+s_a)^2+s_a\right)
e^{|\mu^{(a)}|+s_a/2}
\right\}.$
Then all conclusions of
Theorem~\ref{thm:discrete-mab-horizon-independent-learning-rate} hold whenever
$\ell\leq\min\{1/d,\Delta_{\min}/(d\nu_{G,\bm\sigma})\}$.  More generally, any
predictable baseline satisfying
\eqref{eq:discrete-mab-arm-dependent-exponential-moments} is admissible.  If
$|B_n|\leq B_{\max}$, the same conclusion follows after replacing
$|\mu^{(a)}|$ above by $|\mu^{(a)}|+B_{\max}$.
\end{remark}

\section{Concluding Remarks}
\label{sec:conclusion}
Despite we work under the Brownian diffusions, our analysis is based on constructing Lyapuov functions; hence, it is possible to extend our framework and method to more general stochastic processes. The analysis based on SDEs presented in this paper may open up many interesting research questions for a complete understanding of online policy gradient in diffusion environment. Beyond MABs, the policy gradient can also be applied to contextual bandits and general stochastic control problems. It would be interesting to examine the performance of policy gradient in these more general problems.

The success of policy gradient suggests the potential of ``model-free" learning that only relies on the optimization principle and bypasses statistical principles in devising algorithms. In MAB, the exploration-exploitation tradeoff is naturally entailed by the policy gradient update without further adjustment.  Whereas, the ``instance-free" regret bound still remains an open problem and the order of our regret upper bound in terms of the number of arms $d$ is larger than the typical regret bound for MAB.

\section*{Acknowledgments}
We thank Xuefeng Gao and Jiacheng Zhang for helpful discussion at the early stage of this work. We especially thank Shuaijie Qian for contributing an elegant alternative proof for the two-arm case during our discussion, which is, however, not reflected in this paper. All errors are our own. 

\bibliographystyle{informs2014}
\bibliography{ref}

\ECSwitch


\ECHead{Electronic Companion}
\section{MAB as A Stochastic Control Problem}
\label{sec:problem-formulation}

We formulate the multi-armed bandit under the continuous-time stochastic control
framework. The decision maker (DM) chooses from a collection of arms $\mathcal A=\{1,\ldots,d\}$. Arm
$a\in\mathcal A$ generates cumulative reward $R_t^{(a)}$ satisfying the following stochastic differential equation (SDE):
\begin{equation}
\dd R_t^{(a)}
=
\mu^{(a)}\,\dd t+\sigma^{(a)}\,\dd B_t^{(a)},
\qquad a\in\mathcal A,
\label{eq:prelim-arm-reward}
\end{equation}
where $\bm\mu=(\mu^{(1)},\ldots,\mu^{(d)})^\top \in \mathbb R^d$ and $\bm\sigma = (\sigma^{(1)},\ldots,\sigma^{(d)})^\top \in \mathbb R^d$ stand for the reward rate and volatility of the reward flow, and
$\bm B=(B^{(1)},\ldots,B^{(d)})^\top \in \mathbb R^d$ is a standard Brownian
motion. 

A classical control problem is to choose an action process $A_t\in\mathcal A$, to maximize the long-run average of the cumulative reward
$R^A$ of the selected arms:
\begin{equation}
\sup_{A}
\liminf_{T\to\infty}
\frac{1}{T}\E[R_T^A] \text{ subject to } \dd R_t^A
=
\sum_{a=1}^d
\one_{\{A_t=a\}}\,\dd R_t^{(a)},\ R^A_0 = 0.
\label{eq:mab-controlled-reward}
\end{equation}

The bandit problem is nontrivial because $\bm\mu$ is unknown and must be learned only from the rewards generated by the selected arms. We follow the model-free continuous-time reinforcement-learning framework by \citet{wang2020reinforcement} and the actor-critic learning in \citet{jia2022policypgac} to showcase how stochastic controls in unknown environment can be attacked systematically.\footnote{A partially observed control formulation (cf. \citealt{elliott1995hidden}) could append a filter for the unknown means, but it would require a specified likelihood model and would generally enlarge the state space, causing difficulty in solving the augmented control problem.}

\subsection{Continuous-time actor--critic learning via relaxed controls}
\label{sec:mab-preliminaries}

\citet{wang2020reinforcement} suggest a relaxed control framework to describe DM's learning via trial and error. In particular, DM follows a stochastic policy $\bm\pi_t
=
(\pi_t^{(1)},\ldots,\pi_t^{(d)})^\top
\in\mathcal P^d$, a probability vector that describes the choice probability, and considers a relaxed control problem
\citep{fleming1984stochastic}:
\begin{equation}
\sup_{\bm\pi}\liminf_{T\to\infty} \frac{1}{T}
\E\left[
\widetilde R_T
\right],\text{ subject to } \dd\widetilde R_t
=
\bm\mu^\top\bm\pi_t\,\dd t+ \sqrt{(\bm\sigma^2)^\top \bm\pi_t}\dd B_t^R,\ \widetilde R_0 = 0
\label{eq:mab-relaxed-objective}
\end{equation}
where $B^R$ is a standard Brownian motion, and $\bm\sigma^2\in \mathbb R^d$ stands for the entry-wise square. $\widetilde R_t$ reflects the ``average" trajectory of the accumulative reward by ``integrating out" the randomness in choosing $a_t\sim \bm\pi_t$ ``almost continuously in time". See
\citet{jia2026accuracy} for rigorous account of the relations between SDEs in \eqref{eq:mab-controlled-reward} and \eqref{eq:mab-relaxed-objective}. The relaxed control problem \eqref{eq:mab-relaxed-objective} by itself is still a trivial control
problem. Its formal (ergodic) Hamilton--Jacobi--Bellman equation is
\[
\max_{\bm\pi\in\mathcal P^d}
\left\{
(\bm\mu^\top\bm\pi)V^{\star'}(R)
+
\frac{1}{2}(\bm \sigma^2)^{\top}\bm \pi V^{\star''}(R)
+
\bm\mu^\top\bm\pi-\beta^\star
\right\}
=0,
\]
where the solution is given by $V^\star\equiv0$,
$\beta^\star=\max_{a\in \mathcal A}\mu^{(a)}$, and the optimal policy is to always choose optimal arms.
Whereas, adopting relaxed controls is not to study a new control problem, instead, \citet{jia2022policy,jia2022policypgac,jia2023q} demonstrate that principled model-free learning algorithms can be devised for relaxed stochastic controls. We specialize the online
policy-gradient actor--critic method of \citet{jia2022policypgac} to the MAB.

For this actor-critic learning procedure, one needs to parameterize the value function and the policy separately. Since $V^\star \equiv 0$ is common knowledge that does not depend on $\bm\mu$, we only need to learn the long-run average reward,
represented by a scalar parameter $\beta$. The policy is commonly parameterized by the logit function
\begin{equation*}
\pi^{(a)}(\bm\phi)
=
\frac{\exp(\phi^{(a)})}
{\sum_{j=1}^d\exp(\phi^{(j)})},
\text{ with }
\bm\pi(\bm\phi)
=
(\pi^{(1)}(\bm\phi),\ldots,\pi^{(d)}(\bm\phi))^\top\in \mathcal P^d,\ \bm\phi=(\phi^{(1)},\ldots,\phi^{(d)})^\top\in \mathbb R^d .
\label{eq:prelim-softmax-policy}
\end{equation*}

The derivation in \citet{jia2022policypgac} suggests the learning scheme via stochastic approximation to solve the following martingale conditions and the first order conditions of the policy gradient:
\[ \left\{ \begin{aligned}
& \E\left[ \int (\dd R^A_t + \dd V^\star_t - \beta_t\dd t) \right]  = 0 \text{ (for learning $\beta$) }   \\
& \E\left[ \int \nabla_{\bm\phi}\log\pi^{(A_t)}(\bm\phi_t)(\dd R^A_t + \dd V^\star_t - \beta_t\dd t)  \right] = 0 \text{ (for learning $\bm\phi$) } 
\end{aligned}\right. . \]
Thus, the online, incremental learning can be informally represented by \eqref{eq:mab-continuous-actor-critic}. A common choice is to take the learning rates as $\alpha_t = 1/(t+\eta)$ and $\ell_t \equiv \ell$. Under this choice, we can solve $\beta_t$ in \eqref{eq:mab-continuous-actor-critic} explicitly as
$\beta_t=(R_t^A + \beta_0\eta)/(t+\eta)$; thus $\beta_t$ can be interpreted as the approximate average-reward, and it is used as a natural ``benchmark" in the update of $\bm\phi_t$.

The policy gradient update for $\bm\phi_t$ in \eqref{eq:mab-continuous-actor-critic} is only informal in continuous-time because continuous independent sampling is infeasible. Its implementation is summarized in E-Companion \ref{sec:code}. Following \citet{wang2020reinforcement} and \citet{jia2026accuracy}, we ``integrate out" the impact of randomization of action $a_t\sim \bm\pi_t$ with the high frequency sampling, and the statistical properties of $\bm\phi_t$ is fully characterized by the aggregated SDE \eqref{eq:prelim-aggregated-sde}. We present an intuitive derivation of \eqref{eq:prelim-aggregated-sde} in the E-Companion \ref{sec:intuitive aggregated sde}, and such a derivation is rigorously proved as the weak limit of the high-frequency sampling in \citet{jia2026accuracy}.

\section{Diffusion Limit of the Discrete-Time Policy Gradient Algorithm}
\label{subsec:mab-discrete-implementation}

To connect the continuous-time framework with the conventional discrete-time policy gradient (e.g., \citealt{mei2023stochastic}), we consider the latter in the ``diffusion limit" \citep{fan2021diffusion,kuang2024weak}. To be more precise, upon every draw, arm $a$ generates reward $r_k^{(a),(n)} = \mu^{(a)} + \sqrt{n} \xi_k^{(a)}$, where $n$ is a scaling parameter that measures the noise to signal ratio of the reward, and $\xi_k^{(a)}$ are i.i.d. noise whose distribution is irrelevant of $n$. 

The conventional discrete-time policy gradient updates the policy (with the same logit parameterization and a constant learning rate $\ell$) via:
\begin{equation}
\begin{aligned}
R_{k+1}^{(n)}
&=
R_k^{(n)}+ \Delta R_{k+1}^{(n)} = R_k^{(n)} + r_k^{(A_k^{(n)}),(n)},\ R_{0}^{(n)} = 0, \\
\bm\phi_{k+1}^{(n)}
&=
\bm\phi_k^{(n)}
+
\ell\bigl(\bm e_{A_k^{(n)}}-\bm\pi_k^{(n)}\bigr)
\bigl(\Delta R_{k+1}^{(n)}-\beta_k^{(n)}\bigr),
\end{aligned}
\label{eq:mab-implemented-pg}
\end{equation}
where $\bm\pi_k^{(n)}=\bm\pi(\bm\phi_k^{(n)}/n)$, and $A_k^{(n)}\sim \bm\pi_k^{(n)}$ is an independent random draw. 

Let
$(\widetilde{\bm\phi}_t^{(n)},\widetilde{R}_t^{(n)})$ be scaled piecewise-constant interpolation of this recursion, defined as
\[\widetilde{\bm\phi}_t^{(n)} = \frac{1}{n} \bm\phi_{\lfloor{t n\rfloor}}^{(n)},\ \widetilde{R}_t^{(n)} = \frac{1}{n} R_{\lfloor{t n\rfloor}}^{(n)} .  \]

\begin{theorem}
\label{thm:mab-implemented-pg-weak-limit}
Fix $ T>0$. Let
$\mathcal F_k^{(n)}$ contain the history before round $k$, so that
$\bm\phi_k^{(n)}$, $R_k^{(n)}$, and
$\beta_k^{(n)}$ are $\mathcal F_k^{(n)}$-measurable, and let
$\mathcal F_{k+1}^{(n)}$ contain the sampled action and observed innovation.
Suppose that
$(\widetilde{\bm\phi}_0^{(n)},\widetilde R_0^{(n)})
=(\bm\phi_0,0)$ is deterministic and independent of $n$.  Write
$s_a=(\sigma^{(a)})^2$ and
$\E_{k,a}^{(n)}[\cdot]
:=\E[\cdot\mid\mathcal F_k^{(n)},A_k^{(n)}=a]$.  Assume, uniformly in
$n,k$, and $a$,
\begin{equation}
\E_{k,a}^{(n)}[\xi_k^{(a)}]=0,
\ \
\E_{k,a}^{(n)}[(\xi_k^{(a)})^2]=s_a,
\ \
\E_{k,a}^{(n)}[|\xi_k^{(a)}|^4]\le C_\xi, \ \ \sup_{n\ge1}\max_{0\le k\le\lfloor nT\rfloor}
\E\left[|\beta_k^{(n)}|^4\right]<\infty.
\label{eq:mab-implemented-baseline-condition}
\end{equation}
Then the following statements hold.

\emph{(i) Weak convergence.}  As $n\to\infty$,
\begin{equation}
(\widetilde{\bm\phi}^{(n)},\widetilde{R}^{(n)})
\ \Longrightarrow\
(\bm\phi,\widetilde R)
\quad\text{in}\quad
D([0,T];\mathbb R^{d+1}),
\label{eq:mab-implemented-joint-convergence}
\end{equation}
where $D([0,T];\mathbb R^{d+1})$ is the Skorokhod space and
$(\bm\phi,\widetilde R)$ is the unique (weak) solution of \eqref{eq:prelim-aggregated-sde} and
\eqref{eq:mab-relaxed-objective} with $\dd \bm B^{\bm\phi}_t \dd B^R_t = \frac{1}{\sqrt{(\bm\sigma^2)^\top \bm\pi(\bm\phi_t)}}\left( \sqrt{\pi^{(1)}(\bm\phi_t)\sigma^{(1)^2}},\cdots, \sqrt{\pi^{(d)}(\bm\phi_t)\sigma^{(d)^2}} \right)^\top$.  

\emph{(ii) Weak convergence order.}  For every
$f\in C_b^4(\mathbb R^{d+1})$, i.e., a fourth-time continuously differentiable function whose up to fourth-time order derivatives are bounded,  there is a constant $C_{f,T}<\infty$,
independent of $n$, such that
\begin{equation}
\sup_{0\le t\le T}
\left|
\E f(\widetilde{\bm\phi}_t^{(n)},\widetilde{R}_t^{(n)})
-
\E f(\bm\phi_t,\widetilde R_t)
\right|
\le
C_{f,T}n^{-1/2}.
\label{eq:mab-implemented-half-order-weak-rate}
\end{equation}
If, in addition,
\begin{equation}
\E_{k,a}^{(n)}[(\xi_k^{(a)})^3]=0
\qquad\text{for all }n,k,a,
\label{eq:mab-implemented-third-moment-matching}
\end{equation}
then the weak error is first order:
\begin{equation}
\sup_{0\le t\le T}
\left|
\E f(\widetilde{\bm\phi}_t^{(n)},\widetilde R_t^{(n)})
-
\E f(\bm\phi_t,\widetilde R_t)
\right|
\le
C_{f,T}n^{-1}.
\label{eq:mab-implemented-first-order-weak-rate}
\end{equation}
\end{theorem}

\begin{proof}{Proof.}
Put $h=1/n$, $t_k=kh$,
$\bm X_k^{(n)}
=(n^{-1}(\bm\phi_k^{(n)})^\top,n^{-1}R_k^{(n)})^\top$, and
$\bm X_t^{(n)}=\bm X_{\lfloor nt\rfloor}^{(n)}$.  Write
$\E_k[\cdot]=\E[\cdot\mid\mathcal F_k^{(n)}]$ and
$\bm p=\bm\pi(\widetilde{\bm\phi}_{k/n}^{(n)})$.  For
$\bm u_a=\bm e_a-\bm p$, set
$\bm w_a=(\ell\bm u_a^\top,1)^\top$ and
$\bm d_{a,k}
=(\ell(\mu^{(a)}-\beta_k^{(n)})\bm u_a^\top,\mu^{(a)})^\top$.
Then the scaled recursion has the exact increment
\begin{equation}
\Delta\bm X_{k+1}^{(n)}
=
h\bm d_{A_k^{(n)},k}
+
\sqrt h\,\bm w_{A_k^{(n)}}\xi_k^{(A_k^{(n)})}.
\label{eq:appendix-implemented-scaled-increment}
\end{equation}

For $\bm p\in\mathcal P^d$, define
\[
\bm b(\bm p)
:=
\begin{pmatrix}
\ell\bm J(\bm p)\bm\mu\\
\bm\mu^\top\bm p
\end{pmatrix},
\qquad
\bm\Gamma(\bm p)
:=
\begin{pmatrix}
\ell\bm G_{\bm\sigma}(\bm p)\\
\bm g_{\bm\sigma}(\bm p)^\top
\end{pmatrix},
\qquad
\bm a(\bm p):=\bm\Gamma(\bm p)\bm\Gamma(\bm p)^\top .
\]
The identity $\sum_a p^{(a)}\bm u_a=\bm0$ cancels the baseline in the
conditional drift, and direct conditioning gives
\begin{align}
\E_k[\Delta\bm X_{k+1}^{(n)}]
&=h\bm b(\bm p),
\label{eq:appendix-implemented-first-moment}\\
\operatorname{Cov}_k(\Delta\bm X_{k+1}^{(n)})
&=h\bm a(\bm p)+h^2\bm R_{k,n},
\qquad
\|\bm R_{k,n}\|
\le C(1+|\beta_k^{(n)}|^2),
\label{eq:appendix-implemented-second-moment}\\
\E_k\!\left[\|\Delta\bm X_{k+1}^{(n)}\|^4\right]
&\le Ch^2(1+|\beta_k^{(n)}|^4).
\label{eq:appendix-implemented-fourth-moment}
\end{align}
where
$\bm R_{k,n}
=\sum_a p^{(a)}\bm d_{a,k}\bm d_{a,k}^\top
-\bm b(\bm p)\bm b(\bm p)^\top$.  Moreover,
\begin{equation}
\bm a(\bm p)
=
\begin{pmatrix}
\ell^2\bm Q_{\bm\sigma}(\bm p)
&
\ell\bm J(\bm p)\bm s\\
\ell\bm s^\top\bm J(\bm p)
&
\bm s^\top\bm p
\end{pmatrix},
\qquad
\bm s=(s_1,\ldots,s_d)^\top,
\label{eq:appendix-implemented-joint-covariance}
\end{equation}
Thus $\bm b$ and $\bm a$ are the drift and covariance of the following SDE
\begin{equation}
\begin{aligned}
\dd\bm\phi_t
&=
\ell\bm J(\bm\pi_t)\bm\mu\,\dd t
+
\ell\bm G_{\bm\sigma}(\bm\pi_t)\,\dd\bm W_t,\\
\dd\widetilde R_t
&=
\bm\mu^\top\bm\pi_t\,\dd t
+
\bm g_{\bm\sigma}(\bm\pi_t)^\top\,\dd\bm W_t,
\quad
\bm\pi_t=\bm\pi(\bm\phi_t),\quad (\bm\phi_0,\widetilde R_0)
=
(\bm\phi_0,0),
\end{aligned}
\label{eq:mab-implemented-joint-limit-sde}
\end{equation}
where $ \bm W$ is a standard $d$-dimensional Brownian motion. 
In particular, the actor--reward
cross-covariance is $\ell\bm J(\bm p)\bm s$, which vanishes for common
volatility since then $\bm s\propto\bm e$.

Apply \citet[Corollary~2.2]{ispany2010note}, the random-step-process
specialization of \citet[Theorem~IX.3.39]{jacod2003limit}, with initial term
$\bm X_0^{(n)}$ and increments
$\bm U_{k+1}^{(n)}=\Delta\bm X_{k+1}^{(n)}$.  Put
$N_t=\lfloor nt\rfloor$,
$\bm b_s^{(n)}
=\bm b(\bm\pi(\widetilde{\bm\phi}_s^{(n)}))$, and
$\bm a_s^{(n)}
=\bm a(\bm\pi(\widetilde{\bm\phi}_s^{(n)}))$.
Equations \eqref{eq:appendix-implemented-first-moment},
\eqref{eq:appendix-implemented-second-moment},
\eqref{eq:appendix-implemented-fourth-moment}, and
\eqref{eq:mab-implemented-baseline-condition} give the three required estimates:
\[
\begin{aligned}
\sup_{t\le T}
\left\|\sum_{k<N_t}\E_k[\Delta\bm X_{k+1}^{(n)}]
-\int_0^t\bm b_s^{(n)}\,\dd s\right\|
&\le Ch,\\
\E\sup_{t\le T}
\left\|\sum_{k<N_t}\operatorname{Cov}_k(\Delta\bm X_{k+1}^{(n)})
-\int_0^t\bm a_s^{(n)}\,\dd s\right\|
&\le C_T h,\\
\E\sum_{k<N_T}\E_k\!\left[
\|\Delta\bm X_{k+1}^{(n)}\|^2
\one_{\{\|\Delta\bm X_{k+1}^{(n)}\|>\varepsilon\}}
\right]
&\le\varepsilon^{-2}\sum_{k<N_T}
\E\|\Delta\bm X_{k+1}^{(n)}\|^4
\le C_{\varepsilon,T}h .
\end{aligned}
\]
Thus the three conditions hold uniformly on $[0,T]$ in probability.  The
initial-state condition is exact.  Moreover,
$\bm b(\bm\pi(\cdot))$ and $\bm\Gamma(\bm\pi(\cdot))$ are bounded and globally
Lipschitz, so \eqref{eq:mab-implemented-joint-limit-sde} has a unique weak
solution by \citet[Chapter~5, Theorem~2.9]{karatzas1991brownian}.  Hence
\citet[Corollary~2.2]{ispany2010note} yields
\eqref{eq:mab-implemented-joint-convergence}.

For the weak orders, let $(P_r)_{r\ge0}$ and $\mathcal L$ be the semigroup and
generator of \eqref{eq:mab-implemented-joint-limit-sde}.  Its $C_b^\infty$
coefficients and standard backward-Kolmogorov regularity give, for
$f\in C_b^4$,
$\sup_{r\le T}(\|P_rf\|_{C_b^4}+\|\mathcal L^2P_rf\|_\infty)\le C_{f,T}$.
For a grid time $t_m=mh$, set $g_k=P_{t_m-t_k}f$.  Since
$g_m=f$ and $g_k=P_hg_{k+1}$,
\[
\E f(\bm X_m^{(n)})-P_{t_m}f(\bm X_0^{(n)})
=
\sum_{k=0}^{m-1}
\E\!\left[
\E_k g_{k+1}(\bm X_{k+1}^{(n)})
-P_hg_{k+1}(\bm X_k^{(n)})
\right].
\]
Thus, as in \citet[Theorem~4.1]{jia2026accuracy}, it remains to compare the
one-step moments of the present non-Gaussian recursion.  The remaining third
moment follows from \eqref{eq:appendix-implemented-scaled-increment}:
\begin{equation}
\E_k[(\Delta\bm X_{k+1}^{(n)})^{\otimes3}]
=
h^{3/2}
\sum_{a=1}^dp^{(a)}
\E_{k,a}^{(n)}[(\xi_k^{(a)})^3]\bm w_a^{\otimes3}
+
O\!\left(h^2(1+|\beta_k^{(n)}|^4)\right).
\label{eq:appendix-implemented-third-moment}
\end{equation}
Put $\bm x=\bm X_k^{(n)}$ and
$\bm\delta=\Delta\bm X_{k+1}^{(n)}$.  Conditional Taylor expansion gives
\[
\begin{aligned}
\E_k g(\bm x+\bm\delta)
={}&g(\bm x)
+Dg(\bm x)\cdot\E_k[\bm\delta]
+\frac12D^2g(\bm x):\E_k[\bm\delta^{\otimes2}]\\
&+\frac16D^3g(\bm x):\E_k[\bm\delta^{\otimes3}]
+O\!\left(\|g\|_{C_b^4}\E_k\|\bm\delta\|^4\right).
\end{aligned}
\]
Here ``$:$'' denotes full tensor contraction.  Since
$\E_k[\bm\delta^{\otimes2}]
=\operatorname{Cov}_k(\bm\delta)+(\E_k\bm\delta)^{\otimes2}$,
\eqref{eq:appendix-implemented-first-moment} and
\eqref{eq:appendix-implemented-second-moment} make the linear and quadratic
terms $h\mathcal Lg(\bm x)+O(h^2(1+|\beta_k^{(n)}|^2))$.  Meanwhile, the
semigroup identity
$P_hg-g-h\mathcal Lg
=\int_0^h(h-s)P_s\mathcal L^2g\,\dd s$
is $O(h^2)$.  Hence \eqref{eq:appendix-implemented-third-moment} is the only
possible $h^{3/2}$ contribution; it becomes
$O(h^2(1+|\beta_k^{(n)}|^4))$ under
\eqref{eq:mab-implemented-third-moment-matching}, while
\eqref{eq:appendix-implemented-fourth-moment} bounds the remainder at the same
order.  Uniformly for $g=P_rf$, $0\le r\le T$, we therefore have
\begin{equation}
\left|
\E_k g(\bm X_{k+1}^{(n)})
-
P_hg(\bm X_k^{(n)})
\right|
\le
C_{f,T}(1+|\beta_k^{(n)}|^4)
\begin{cases}
h^{3/2},&
\text{in general},\\
h^2,&
\text{under \eqref{eq:mab-implemented-third-moment-matching}},
\end{cases}
\label{eq:appendix-implemented-local-weak-error}
\end{equation}
which is the only use of the third-moment condition.  Summing at most $T/h$
terms and using \eqref{eq:mab-implemented-baseline-condition} gives grid-time
errors $C_{f,T}h^{1/2}$ and $C_{f,T}h$, respectively.  Between grid points
$\bm X^{(n)}$ is constant and
$|P_tf(\bm x)-P_sf(\bm x)|\le\|\mathcal Lf\|_\infty|t-s|$ for
$|t-s|\le h$.  Taking the supremum over $[0,T]$ and using $h=1/n$ proves
\eqref{eq:mab-implemented-half-order-weak-rate} and
\eqref{eq:mab-implemented-first-order-weak-rate}. \Halmos
\end{proof}

The weak convergence of the scaled discrete-time policy gradient algorithm is established using a semimartingale convergence theorem from \citet{jacod2003limit}. Similar diffusion limit of other classical bandit algorithms have been studied by \citet{fan2021diffusion,kuang2024weak} but it is new for the policy gradient. In addition to the weak convergence, we also obtain the rate of convergence in terms of bounded test functions, which is derived for a discrete sampling of SDEs in \citet{jia2026accuracy}.  The obtained aggregated SDE \eqref{eq:prelim-aggregated-sde} is consistent with the one suggested in \citet{lattimore2026diffusion}, thus, we provide a solid micro-foundation for this stochastic system and the framework by \citet{wang2020reinforcement,jia2022policypgac}. 

\section{Pseudo Code for Actor-Critic Algorithm For MAB}
\label{sec:code}
We summarize the implementation of the algorithm \eqref{eq:mab-continuous-actor-critic} by discrete sampling as Algorithm \ref{algo:mab-sampled-actor-critic}. It turns out to coincide with the conventional (discrete-time) policy gradient algorithm with a particular baseline.
\begin{algorithm}[H]
\caption{Online Actor--Critic Algorithm for Multi-Armed Bandit}
\label{algo:mab-sampled-actor-critic}
\begin{algorithmic}[1]
\STATE \textbf{Input:} intervention times
$0=t_0<t_1<\cdots$, actor learning rates $\ell_k$, and critic learning rates
$\alpha_k$.
\STATE Initialize $\bm\phi_0$, $\beta_0$, and $R_0=0$.
\FOR{$k=0,1,\ldots$}
\STATE Set $h_k=t_{k+1}-t_k$ and
$\bm\pi_k=\bm\pi(\bm\phi_k)$.
\STATE Draw $A_k\sim\bm\pi_k$, hold arm $A_k$ on
$[t_k,t_{k+1})$, and observe its reward increment $\Delta R_{k+1}$.
\STATE Update the actor using the pre-update critic,
\[
\bm\phi_{k+1}
=
\bm\phi_k
+
\ell_k(\bm e_{A_k}-\bm\pi_k)
\bigl(\Delta R_{k+1}-\beta_kh_k\bigr).
\]
\STATE Update the critic and accumulated reward,
\[
\beta_{k+1}
=
\beta_k+\alpha_k
\bigl(\Delta R_{k+1}-\beta_kh_k\bigr),
\qquad
R_{k+1}=R_k+\Delta R_{k+1}.
\]
\ENDFOR
\end{algorithmic}
\end{algorithm}



\section{Intuitive Derivation of Policy Gradient SDE \eqref{eq:prelim-aggregated-sde}}
\label{sec:intuitive aggregated sde}

For reader's convenience and pedagogical purpose, we use the heuristic argument in \citet{wang2020reinforcement} to demonstrate how to obtain the aggregated SDE \eqref{eq:prelim-aggregated-sde} from its informal counterpart \eqref{eq:mab-continuous-actor-critic}. The rigorous argument can be found in \citet{jia2026accuracy} for more general SDEs. 

The $\dd t$ term in $\dd \bm\phi_t$ in \eqref{eq:mab-continuous-actor-critic} is $ \ell_t \sum_{a=1}^d \one_{\{ A_t = a\}} (\bm e_{a} - \bm\pi(\bm\phi_t))  (\mu^{(a)}_t - \beta_t)$. We take the expectation of this term with respect to $A_t\sim \bm\pi(\bm\phi_t)$ conditioned on $\bm\phi_t,\beta_t$, we get 
\[ \begin{aligned}
& \E_{A_t\sim \bm\pi(\bm\phi_t)}\left[  \ell_t \sum_{a=1}^d \one_{\{ A_t = a\}} (\bm e_{a} - \bm\pi(\bm\phi_t))  (\mu^{(a)}_t - \beta_t) \right] \\
= & \ell_t \left[  \diag\{ \bm\pi(\bm\phi_t) \}(\bm\mu - \beta_t \bm e) - (\bm\pi(\bm\phi_t)^\top \bm\mu - \beta_t)\bm\pi(\bm\phi_t)  \right] = \ell_t \left[\diag\{ \bm\pi(\bm\phi_t) \} -  \bm\pi(\bm\phi_t)\bm\pi(\bm\phi_t)^\top \right] \bm \mu \\
= & \ell_t \bm J(\bm\pi_t) \bm\mu.  
\end{aligned}   \]

The $\dd B_t$ term in $\dd \bm\phi_t$ in \eqref{eq:mab-continuous-actor-critic} is $ \ell_t \sum_{a=1}^d \one_{\{ A_t = a\}} (\bm e_{a} - \bm\pi(\bm\phi_t)) \sigma^{(a)}_t\dd B^{(a)}_t$. We take the expectation to the quadratic variation of this term with respect to $A_t\sim \bm\pi(\bm\phi_t)$ conditioned on $\bm\phi_t,\beta_t$, we get  
\[ \begin{aligned}
& \E_{A_t\sim \bm\pi(\bm\phi_t)}\left[ \dd \bm\phi_t \dd \bm\phi_t^\top \mid \bm\phi_t,\beta_t  \right] \\
= & \ell_t^2\E_{A_t\sim \bm\pi(\bm\phi_t)}\left[ \sum_{a=1}^d \one_{\{A_t = a\}} \sigma^{(a)^2}(\bm e_a - \bm\pi(\bm\phi_t))(\bm e_a - \bm\pi(\bm\phi_t))^\top \right] \dd t \\
= & \ell_t^2 \left[ \diag\{ \bm\pi(\bm\phi_t) \}\diag\{\bm\sigma^2\}  - \diag\{ \bm\pi(\bm\phi_t) \} \bm\sigma^2    \bm\pi(\bm\phi_t)^\top - \bm\pi(\bm\phi_t) \bm\sigma^{2^\top}\diag\{ \bm\pi(\bm\phi_t) \}  + (\bm\pi(\bm\phi_t)^\top \bm\sigma^2) \bm\pi(\bm\phi_t) \bm\pi(\bm\phi_t)^\top \right]\dd t \\
= & \ell_t^2 (\bm I-\bm\pi\bm e^\top)
\diag\{\pi^{(1)} \sigma^{(1)^2},\ldots,\pi^{(d)}\sigma^{(d)^2}\}
(\bm I-\bm\pi\bm e^\top)^\top \dd t.
\end{aligned} \]

The expected drift and quadratic variation (taking the expectation with respect to $A_t$) of $\dd \bm\phi_t$ in \eqref{eq:mab-continuous-actor-critic} coincides that in the aggregated SDE \eqref{eq:prelim-aggregated-sde} with constant learning rate $\ell_t \equiv \ell$. The relations between SDEs in \eqref{eq:mab-controlled-reward} and \eqref{eq:mab-relaxed-objective} can be similarly obtained.

The above investigation only restricts to $\bm\phi_t$ and $R_t$ separately. We can further examine their cross variation by integrating out $A_t\sim \bm\pi(\bm\phi_t)$ conditioned on $\bm\phi_t,\beta_t$, that is,
\[ \begin{aligned}
& \E_{A_t\sim \bm\pi(\bm\phi_t)}\left[ \dd R^{A}_t \dd \bm\phi_t  \mid \bm\phi_t,\beta_t  \right] \\
= & \ell_t \E_{A_t\sim \bm\pi(\bm\phi_t)}\left[ \sum_{a=1}^d \one_{\{A_t = a\}} \sigma^{(a)^2}(\bm e_a - \bm\pi(\bm\phi_t)) \right] \dd t \\
= & \ell_t \left[ \diag\{\bm\pi(\bm\phi_t)\} \bm\sigma^2 - \bm\pi(\bm\phi_t)^\top\bm\sigma^2 \bm\pi(\bm\phi_t)\right] \dd t = \ell_t \bm J(\bm\pi_t) \bm\sigma^2 \dd t.
\end{aligned} \]

The cross-variation structure motivates the correlation between $\dd B_t^R$ and $\dd \bm B^{\bm\phi}_t$ in \eqref{eq:mab-relaxed-objective} and \eqref{eq:prelim-aggregated-sde}.

\section{Intuition from Two-Armed Bandit}
\label{sec:two arm}
To gain some intuition about the aggregated SDE \eqref{eq:prelim-aggregated-sde}, we look at the special case of two-armed bandit $d=2$. In this case, the policy can be represented by $\pi_t^{(1)} = \frac{e^{\phi_t^{(1)}}}{e^{\phi_t^{(1)}} + e^{\phi_t^{(2)}}} = \frac{e^{\phi_t^{(1)} - \phi_t^{(2)}}}{e^{\phi_t^{(1)} - \phi_t^{(2)}} + 1}$, and $\pi_t^{(2)} = \frac{1}{e^{\phi_t^{(1)} - \phi_t^{(2)}} + 1}$. Hence, it suffices to examine the property of $\delta\phi_t = \phi_t^{(1)} - \phi_t^{(2)}\in \mathbb R$, a scalar variable; or equivalently $\pi_t^{(1)} = \frac{e^{\delta \phi_t}}{e^{\delta\phi_t} + 1}$. 

Then \eqref{eq:prelim-aggregated-sde} reduces to
\[\begin{aligned}
\dd \delta\phi_t = & 2\ell \pi_t^{(1)}(1 - \pi_t^{(1)})(\mu^{(1)} - \mu^{(2)}) \dd t + 2\ell \sqrt{\pi_t^{(1)} (1 - \pi_t^{(1)})} \left( \sqrt{1-\pi_t^{(1)}}\sigma^{(1)},-\sqrt{\pi_t^{(1)}}\sigma^{(2)}  \right) \dd \bm B_t^{\bm\phi} \\
\stackrel{d}{=} & 2\ell \pi_t^{(1)}(1 - \pi_t^{(1)})(\mu^{(1)} - \mu^{(2)}) \dd t + 2\ell \sqrt{\pi_t^{(1)} (1 - \pi_t^{(1)})}\sqrt{(1-\pi_t^{(1)})\sigma^{(1)^2} + \pi_t^{(1)}\sigma^{(2)^2}}\dd B_t ,
\end{aligned}     \]
where ``$\stackrel{d}{=} $" means equality in distribution.

We notice that the direction of the drift in the parameter space ($\delta\phi_t$) is a constant and governed by the difference of drift $\mu^{(1)} - \mu^{(2)}$ and its magnitude is determined by the learning rate $\ell$. Since $\delta\phi_t$ is a scalar, the volatility part can be equivalently represented by a scalar Brownian motion $B_t$, and it also scales with the learning rate $\ell$. Such a structure does not immediately imply $\delta\phi_t$ would grow to infinity caused by the drift $\mu^{(1)} - \mu^{(2)}$ as such a force would be canceled with the term $\pi_t^{(1)}(1 - \pi_t^{(1)})$ which is diminishing as $\delta\phi_t$ approaches to infinity.  

Therefore, we have to examine the dynamic of the policy $\pi_t^{(1)}$ in the policy space. By It\^o's lemma, we have
\begin{equation}
\label{eq:policy sde two arm}
\begin{aligned}
\dd \pi_t^{(1)} = & \pi_t^{(1)}(1 - \pi_t^{(1)}) \dd \delta\phi_t + \frac{1}{2}\pi_t^{(1)}(1 - \pi_t^{(1)})(1-2\pi_t^{(1)})\dd \langle \delta\phi \rangle_t  \\
= & 2\ell \left( \pi_t^{(1)}(1 - \pi_t^{(1)}) \right)^2 (\mu^{(1)} - \mu^{(2)}) \dd t + 2\ell \left( \pi_t^{(1)} (1 - \pi_t^{(1)})\right)^{3/2}\sqrt{(1-\pi_t^{(1)})\sigma^{(1)^2} + \pi_t^{(1)}\sigma^{(2)^2}} \dd B_t \\
& + \underbrace{2\ell^2\left( \pi_t^{(1)}(1 - \pi_t^{(1)}) \right)^2(1- 2\pi_t^{(1)}) \left[(1-\pi_t^{(1)})\sigma^{(1)^2} + \pi_t^{(1)}\sigma^{(2)^2}  \right]  \dd t}_{\text{It\^o's correction term}} \\
= & \underbrace{2\ell \left( \pi_t^{(1)}(1 - \pi_t^{(1)}) \right)^2}_{\text{how fast learning vanishes}} \left[ \underbrace{\mu^{(1)} - \mu^{(2)}}_{\text{desired driving force}} + \underbrace{\ell (1- 2\pi_t^{(1)} ) \left( (1-\pi_t^{(1)})\sigma^{(1)^2} + \pi_t^{(1)}\sigma^{(2)^2}   \right)}_{\text{distortion in signal}} \right] \dd t \\
& + \underbrace{2\ell \left( \pi_t^{(1)} (1 - \pi_t^{(1)})\right)^{3/2}}_{\text{how fast volatility vanishes}} \underbrace{\sqrt{(1-\pi_t^{(1)})\sigma^{(1)^2} + \pi_t^{(1)}\sigma^{(2)^2}}}_{\text{bounded from below and above}} \dd B_t .
\end{aligned} 
\end{equation}

On the one hand, the leading direction of the dynamic of the policy $\pi_t^{(1)}$ is governed by two forces: the reward rate gap $\mu^{(1)} - \mu^{(2)}$, which is the desired signal indicating the better arm, and the distortion term $\ell (1- 2\pi_t^{(1)} ) \left( (1-\pi_t^{(1)})\sigma^{(1)^2} + \pi_t^{(1)}\sigma^{(2)^2}   \right)$, caused by the propagation of noises. Any limiting point of \eqref{eq:policy sde two arm}, if exists, must enforce the drift of $\pi_t^{(1)}$ to be 0. Besides two obvious absorbing points 0 and 1, the distortion term may cause another point of such an equilibrium point, that value $\hat\pi$ such that $\mu^{(1)} - \mu^{(2)} + \ell (1- 2\hat\pi) \left( (1-\hat\pi)\sigma^{(1)^2} + \hat\pi\sigma^{(2)^2}   \right) = 0$. If $\hat\pi\notin (0,1)$, then it does not affect the limiting behavior; otherwise if $\hat\pi\in (0,1)$, it becomes an undesired equilibrium point. Despite $\hat\pi$ is not an absorbing point because there is still random noise that will drive $\pi_t^{(1)}$ away from $\hat\pi$, it may still slow down the convergence rate of the process. The value of $\hat\pi$ can be controlled by the learning rate $\ell$. When $\ell$ is sufficiently small, it is guaranteed that $\hat\pi\notin (0,1)$, thus, we can expect faster convergence. This intuition is consistent with the small learning rate condition in Theorem \ref{thm:mab-heterogeneous-volatility-regret}.

On the other hand, between two absorbing points 0 and 1, whether it is attainable within a finite time is largely determined by the relative magnitude of the drift and volatility as $\pi_t^{(1)}$ approaches to 0 or 1. In the one-dimension case, there is the well-known Feller's test for explosion which describes a rate function that serves as the Lyapunov function to analyze the limiting behavioral, see, e.g., \citet[Chapter 5, Proposition 5.22]{karatzas1991brownian}. Applying the conclusion therein to \eqref{eq:policy sde two arm}, one can prove the almost sure convergence for arbitrary constant learning rate. This conclusion is a special case of Theorem \ref{thm:mab-heterogeneous-arbitrary-learning-rate-as-convergence}.


\end{document}